\documentclass[twoside]{article}

\usepackage[accepted]{aistats2023}
\usepackage[round]{natbib}

\usepackage{silence}
\usepackage[utf8]{inputenc} 
\usepackage[T1]{fontenc}    
\usepackage{url}            
\usepackage{booktabs}       
\usepackage{amsfonts}       
\usepackage{nicefrac}       
\usepackage{microtype}      
\usepackage{xcolor}         
\usepackage{algorithm}
\usepackage[noend]{algorithmic}
\usepackage{bm}
\usepackage{enumitem}

\usepackage{amsmath}
\usepackage{amssymb}
\usepackage{mathtools}
\usepackage{amsthm}
\usepackage[]{amsbsy}
\usepackage{commath}
\usepackage{thm-restate} 
\usepackage{physics} 
\usepackage{float} 
\usepackage[caption=false,font=footnotesize]{subfig}
\usepackage[skip=0pt]{caption}

\usepackage{graphicx}
\graphicspath{{figures/}} 

\usepackage[toc,page,header]{appendix}
\usepackage{minitoc}

\theoremstyle{plain}

\newtheorem{lemma}{Lemma}
\newtheorem{corollary}{Corollary}
\newtheorem{assumption}{Assumption}
\theoremstyle{definition}

\theoremstyle{remark}

\usepackage[hidelinks]{hyperref}       
\usepackage[capitalize,noabbrev]{cleveref}

\crefname{assumption}{Assumption}{Assumptions}
\crefname{theorem}{Theorem}{Theorems}
\crefname{equation}{}{}
\crefname{ALC@unique}{Line}{Lines}

\let\originalleft\left
\let\originalright\right
\renewcommand{\left}{\mathopen{}\mathclose\bgroup\originalleft}
\renewcommand{\right}{\aftergroup\egroup\originalright}

\newcommand{\paren}[1]{\left(#1\right)}
\newcommand{\bracket}[1]{\left[#1\right]}
\newcommand{\Cov}{\mathrm{Cov}}
\newcommand{\given}{\;\middle|\;}

\usepackage{etoolbox}
\newcounter{myalg}
\AtBeginEnvironment{algorithmic}{\refstepcounter{myalg}}
\makeatletter
\@addtoreset{ALC@unique}{myalg}
\makeatother

\DeclareMathOperator*{\argmax}{argmax} 

\DeclareMathOperator{\E}{\mathbb{E}}
\DeclareMathOperator{\V}{\mathbb{V}}
\DeclareMathOperator{\R}{\mathbb{R}}

\renewcommand{\paragraph}[1]{\textbf{#1}\hspace{0em}}

\begin{document}

%

%
\runningauthor{Carlos E. Luis, Alessandro G. Bottero, Julia Vinogradska, Felix Berkenkamp, Jan Peters}

\twocolumn[

\aistatstitle{Model-Based Uncertainty in Value Functions}

\aistatsauthor{Carlos E. Luis$^{1,2}$ \And Alessandro G. Bottero$^{1,2}$ \And Julia
Vinogradska$^{1}$ \And Felix Berkenkamp$^{1}$\And Jan Peters$^{2,3,4}$}

\aistatsaddress{$^1$Bosch Center for Artificial Intelligence \quad $^2$Institute for
Intelligent Autonomous Systems, TU Darmstadt  \\
$^{3}$ German Research Center for AI (DFKI), Research Department: Systems AI for Robot Learning
\quad $^{4}$ Hessian.AI}]

\doparttoc 
\faketableofcontents 

\begin{abstract}
  We consider the problem of quantifying uncertainty over expected cumulative rewards in model-based
  reinforcement learning. In particular, we focus on characterizing the \emph{variance} over values
  induced by a distribution over MDPs. Previous work upper bounds the posterior variance over values
  by solving a so-called uncertainty Bellman equation, but the over-approximation may result in
  inefficient exploration. We propose a new uncertainty Bellman equation whose solution converges to
  the true posterior variance over values and explicitly characterizes the gap in previous work.
  Moreover, our uncertainty quantification technique is easily integrated into common exploration
  strategies and scales naturally beyond the tabular setting by using standard deep reinforcement
  learning architectures. Experiments in difficult exploration tasks, both in tabular and continuous
  control settings, show that our sharper uncertainty estimates improve sample-efficiency.
\end{abstract}

\section{INTRODUCTION}
\label{sec:introduction}
The goal of reinforcement learning (RL) agents is to maximize the expected return via interactions
with an \emph{a priori} unknown environment \citep{sutton_reinforcement_2018}. In model-based RL
(MBRL), the agent learns a statistical model of the environment, which can then be used for
efficient exploration \citep{sutton_dyna_1991,strehl_analysis_2008,jaksch_near-optimal_2010}. The
performance of deep MBRL algorithms was historically lower than that of model-free methods, but the
gap has been closing in recent years \citep{janner_when_2019}. Key to these improvements are models
that quantify epistemic and aleatoric uncertainty
\citep{depeweg_decomposition_2018,chua_deep_2018} and algorithms that leverage model uncertainty
to optimize the policy \citep{curi_efficient_2020}. Still, a core challenge in MBRL is to quantify
the uncertainty in long-term performance predictions of a policy given a probabilistic model of
the dynamics \citep{deisenroth_pilco_2011}. Leveraging predictive uncertainty of the policy
performance during policy optimization facilitates \emph{deep exploration} --- methods that reason
about the long-term information gain of rolling out a policy --- which has shown promising results
in the model-free \citep{osband_deep_2016,ciosek_better_2019} and model-based settings
\citep{deisenroth_pilco_2011,fan_model-based_2021}.

We adopt a Bayesian perspective on RL to characterize uncertainty in the decision process via a
posterior distribution. This distributional perspective of the RL environment induces distributions
over functions of interest for solving the RL problem, e.g., the \emph{expected return} of a policy,
also known as the value function. This perspective differs from \emph{distributional} RL
\citep{bellemare_distributional_2017}, whose main object of study is the distribution of the
\emph{return} induced by the inherent stochasticity of the MDP and the policy. As such,
distributional RL models \emph{aleatoric} uncertainty, whereas Bayesian RL focuses on the
\emph{epistemic} uncertainty arising from finite data of the underlying MDP. Recent work by
\citet{eriksson_sentinel_2022} and \citet{moskovitz_tactical_2021} combines Bayesian and
distributional RL for various risk measures accounting for both sources of uncertainty.

We focus on model-based Bayesian RL, where the value distribution is induced by a posterior over
MDPs. In particular, we analyze the \emph{variance} of such a distribution of values.
\citet{schneegass_uncertainty_2010} estimate uncertainty in value functions using statistical
uncertainty propagation, with the caveat of assuming the value distribution is Gaussian. Previous
results by \cite{odonoghue_uncertainty_2018,zhou_deep_2020} establish upper-bounds on the posterior
variance of the values by solving a so-called uncertainty Bellman equation (UBE). These results make
no assumptions on the value distribution and are amenable for deep RL implementations. However,
these bounds over-approximate the variance of the values and thus may lead to inefficient
exploration when used for uncertainty-aware optimization (e.g., risk-seeking or risk-averse
policies). In principle, tighter uncertainty estimates have the potential to improve
data-efficiency, which is the main motivation behind this paper.

\paragraph{Our contribution.}
We show that, under the same assumptions as previous work, the posterior variance of the value
function obeys a Bellman-style recursion \emph{exactly}. Our theory characterizes the gap in the
previously tightest upper-bound by \citet{zhou_deep_2020}, which ignores the inherent aleatoric
uncertainty of acting in a potentially stochastic MDP. Inspired by this insight, we propose
\emph{learning} the solution to the Bellman recursion prescribed by our theory, as done by
\citet{odonoghue_uncertainty_2018}, but integrate it within an actor-critic framework for continuous
action problems, rather than using DQN \citep{mnih_playing_2013} for discrete action selection. Our
experiments in tabular and continuous control problems demonstrate that our variance estimation
method improves sample efficiency when used for optimistic optimization of the policy. The source
code is
available\footnote{\href{https://github.com/boschresearch/ube-mbrl}{\texttt{https://github.com/boschresearch/ube-mbrl}}}.

\paragraph{Related work.}
Model-free approaches to Bayesian RL directly model the distribution over values, e.g., with
normal-gamma priors \citep{dearden_bayesian_1998}, Gaussian Processes \citep{engel_bayes_2003} or
ensembles of neural networks \citep{osband_deep_2016}. \citet{jorge_inferential_2020} estimate value
distributions using a backwards induction framework, while \citet{metelli_propagating_2019}
propagate uncertainty using Wasserstein barycenters. \citet{fellows_bayesian_2021} showed that, due
to bootstrapping, model-free Bayesian methods infer a posterior over Bellman operators rather than
values.

Model-based Bayesian RL maintains a posterior over plausible MDPs given the available data, which induces a distribution over values. The MDP uncertainty is typically represented in the
one-step transition model as a by-product of model-learning. For instance, the well-known PILCO
algorithm by \citet{deisenroth_pilco_2011} learns a Gaussian Process (GP) model of the transition
dynamics and integrates over the model's total uncertainty to obtain the expected values. In order
to scale to high-dimensional continuous-control problems, \citet{chua_deep_2018} propose PETS, which
uses ensembles of probabilistic neural networks (NNs) to capture both aleatoric and epistemic
uncertainty as first proposed by \citet{lakshminarayanan_simple_2017}. Both approaches propagate
model uncertainty during policy evaluation and improve the policy via greedy exploitation over this
model-generated noise. Dyna-style \citep{sutton_dyna_1991} actor-critic algorithms have been paired
with model-based uncertainty estimates for improved performance in both online
\citep{buckman_sample-efficient_2018,zhou_efficient_2019} and offline
\citep{yu_mopo_2020,kidambi_morel_2020} RL.

To balance exploration and exploitation, provably-efficient RL algorithms based on \emph{optimism in
the face of the uncertainty} (OFU) \citep{auer_logarithmic_2006,jaksch_near-optimal_2010} rely on
building upper-confidence (optimistic) estimates of the true values. These optimistic values
correspond to a modified MDP where the rewards are enlarged by an uncertainty bonus, which
encourages exploration. In practice, however, the aggregation of optimistic rewards may severely
over-estimate the true values, rendering the approach inefficient \citep{osband_why_2017}.
\citet{odonoghue_uncertainty_2018} show that methods that approximate the variance of the values can
result in much tighter upper-confidence bounds, while \citet{ciosek_better_2019} demonstrate their
use in complex continuous control problems. Similarly, \citet{chen_ucb_2017} propose a model-free
ensemble-based approach to estimate the variance of values.

Interest about the higher moments of the \emph{return} of a policy dates back to the work of
\citet{sobel_variance_1982}, showing these quantities obey a Bellman equation. Methods that leverage
these statistics of the return are known as \emph{distributional} RL
\citep{tamar_temporal_2013,bellemare_distributional_2017}. Instead, we focus specifically on
estimating and using the \emph{variance} of the \emph{expected return} for policy optimization. A
key difference between the two perspectives is the type of uncertainty they model: distributional RL
models the \emph{aleatoric} uncertainty about the returns, which originates from the aleatoric noise
of the MDP transitions and the stochastic policy; our perspective studies the \emph{epistemic}
uncertainty about the value function, due to incomplete knowledge of the MDP. Provably efficient RL
algorithms use this isolated epistemic uncertainty as a signal to balance exploring the environment
and exploiting the current knowledge.

\citet{odonoghue_uncertainty_2018} propose a UBE whose fixed-point solution converges to a
guaranteed upper-bound on the posterior variance of the value function in the tabular RL setting.
This approach was implemented in a model-free fashion using the DQN \citep{mnih_playing_2013}
architecture and showed performance improvements in Atari games. Follow-up work by
\citet{markou_bayesian_2019} empirically shows that the upper-bound is loose and the resulting
over-approximation of the variance impacts negatively the regret in tabular exploration problems.
\citet{zhou_deep_2020} propose a modified UBE with a tighter upper-bound on the value function,
which is then paired with proximal policy optimization (PPO) \citep{schulman_proximal_2017} in a
conservative on-policy model-based approach to solve continous-control tasks. We propose a new UBE
and integrate it within a model-based soft actor-critic \citep{haarnoja_soft_2018} architecture
similar to \citet{janner_when_2019,froehlich_-policy_2022}.

\section{PROBLEM STATEMENT}
\label{sec:problem_statement}
We consider an agent that acts in an infinite-horizon MDP $\mathcal{M} = \set{\mathcal{S},
\mathcal{A}, p, \rho, r, \gamma}$ with finite state space $\abs{\mathcal{S}} = S$, finite action
space $\abs{\mathcal{A}} = A$, unknown transition function $p: \mathcal{S} \times \mathcal{A} \to
\Delta(S)$ that maps states and actions to the $S$-dimensional probability simplex, an initial state
distribution $\rho: \mathcal{S} \to [0,1]$, a known and bounded reward function $r: \mathcal{S}
\times \mathcal{A} \to \R$, and a discount factor $\gamma \in [0,1)$. Although we consider a known
reward function, the main theoretical results can be easily extended to the case where it is learned
alongside the transition function (see \cref{app:unknown_rewards}). The one-step dynamics $p(s' \mid
s,a)$ denote the probability of going from state $s$ to state $s'$ after taking action $a$. In
general, the agent selects actions from a stochastic policy $\pi: \mathcal{S} \to \Delta(A)$ that
defines the conditional probability distribution $\pi(a\mid s)$. At each time step of episode $t$
the agent is in some state $s$, selects an action $a \sim \pi(\cdot \mid s)$, receives a reward
$r(s,a)$, and transitions to a next state $s' \sim p(\cdot \mid s,a)$. We define the value function
$V^{\pi,p}: \mathcal{S} \to \R$ of a policy $\pi$ and transition function $p$ as the expected sum of
discounted rewards under the MDP dynamics,
\begin{equation}
  V^{\pi, p}(s) = \E_{\tau \sim P}\bracket{\sum\nolimits_{h=0}^{\infty}\gamma^hr(s_h,a_h) \given s_0 = s},
\end{equation}
where the expectation is taken under the random trajectories $\tau$ drawn from the trajectory
distribution $P(\tau) = \prod_{h=0}^{\infty} \pi(a_h \mid s_h)p(s_{h+1}\mid s_h, a_h)$.

We consider a Bayesian setting similar to previous work by
\citet{odonoghue_uncertainty_2018,odonoghue_variational_2021,zhou_deep_2020}, in which the
transition function $p$ is a random variable with some known prior distribution $\Phi_0$. Define the
transition data observed up to episode $t$ as $\mathcal{D}_t$, then we update our belief about the
random variable $p$ by applying Bayes' rule to obtain the posterior distribution conditioned on
$\mathcal{D}_t$, which we denote as $\Phi_t$. The distribution of transition functions naturally
induces a distribution over value functions. The main focus of this paper is to study methods that
estimate he \emph{variance} of the value function $V^{\pi,p}$ under the posterior distribution
$\Phi_t$, namely $\V_{p \sim \Phi_t} \bracket{V^{\pi, p}(s)}$. Our theoretical results extend to
state-action value functions (see \cref{app:extension_q_values}). The motivation behind studying
this quantity is its potential use for exploring the environment.

\citet{zhou_deep_2020} introduce a method to upper-bound the variance of $Q$-values by solving a UBE. Their theory holds for a class of MDPs where the value functions and transition functions are uncorrelated. This family of MDPs is characterized by the following assumptions:
\begin{assumption}[Independent transitions]
    \label{assumption:transitions}
    $p(s' \mid x,a)$ and $p(s' \mid y,a)$ are independent random variables if $x\neq y$.
\end{assumption}
\begin{assumption}[Acyclic MDP \citep{odonoghue_uncertainty_2018}]
    \label{assumption:acyclic}
    The MDP $\mathcal{M}$ is a directed acyclic graph, i.e., states are not visited more than once
    in any given episode.
\end{assumption}
\Cref{assumption:transitions} holds naturally in the case of discrete state-action spaces with a
tabular transition function, where there is no generalization. \Cref{assumption:acyclic} is
non-restrictive as any finite-horizon MDP with cycles can be transformed into an equivalent
time-inhomogeneous MDP without cycles by adding a time-step variable $h$ to the state-space.
Similarly, for infinite-horizon MDPs we can consider an effective horizon $H = 1/ 1 - \gamma$ and
apply the same logic. The direct consequence of these assumptions is that the random variables
$V^{\pi,p}(s')$ and $p(s' \mid s,a)$ are uncorrelated (see
\cref{lemma:independence_from_assumptions,lemma:uncorrelated_property} in \cref{app:proofs_thm_ube} for a formal proof). 

Other quantities of interest are the posterior mean transition function starting from the current
state-action pair $(s,a)$, 
\begin{equation}
  \label{eq:mean_model}
  \bar{p}_t(\cdot \mid s,a) = \E_{p \sim \Phi_t}\bracket{p(\cdot \mid s,a)}, 
\end{equation}
and the posterior mean value function for any $s \in \mathcal{S}$,
\begin{equation}
  \label{eq:mean_value_function}
  \bar{V}_t^{\pi}(s) = \E_{p \sim \Phi_t}\bracket{V^{\pi,p}(s)},
\end{equation}
where the subscript $t$ represents the dependency on $\mathcal{D}_t$ of both quantities. Note that
$\bar{p}_t$ is a transition function that combines both aleatoric \emph{and} epistemic uncertainty.
Even if we limit the posterior $\Phi_t$ to only include deterministic transition functions,
$\bar{p}_t$ remains a stochastic transition function due to the epistemic uncertainty.

\citet{zhou_deep_2020} define the \emph{local} uncertainty
\begin{equation}
  w_t(s) = \V_{p\sim \Phi_t}
  \bracket{\sum\nolimits_{a, s'}\pi(a \mid s) p(s' \mid s,a) \bar{V}^\pi_t(s')},
\end{equation}
and solve the UBE
\begin{equation}
  \label{eq:ube_pombu}
  W_t^\pi(s) = 
  \gamma ^ 2w_t(s) + \gamma^2\sum_{a, s'}\pi(a \mid s)\bar{p}_t(s' \mid s,a) W_t^\pi(s'),
\end{equation}
whose unique solution satisfies $W_t^\pi \geq \V_{p \sim \Phi_t} \bracket{V^{\pi, p}(s)}$.

\section{UNCERTAINTY BELLMAN EQUATION}
In this section, we build a new UBE whose fixed-point solution is \emph{equal} to the variance of
the value function and we show explicitly the gap between \cref{eq:ube_pombu} and $\V_{p \sim
\Phi_t}\bracket{V^{\pi, p}(s)}$.

The values $V^{\pi,p}$ are the fixed-point solution to the Bellman expectation equation, which
relates the value of the current state $s$ with the value of the next state $s'$. Further, under
\cref{assumption:transitions,assumption:acyclic}, applying the expectation operator to the Bellman
recursion results in $\bar{V}_t^{\pi}(s) = V^{\pi, \bar{p}_t}(s)$. The Bellman recursion propagates
knowledge about the \emph{local} rewards $r(s,a)$ over multiple steps, so that the value function
encodes the \emph{long-term} value of states if we follow policy $\pi$. Similarly, a UBE is a
recursive formula that propagates a notion of \emph{local uncertainty}, $u_t(s)$, over multiple
steps. The fixed-point solution to the UBE, which we call the $U$-values, encodes the \emph{long-term
epistemic uncertainty} about the values of a given state.

Previous formulations by \citet{odonoghue_uncertainty_2018,zhou_deep_2020} differ only on their
definition of the local uncertainty and result on $U$-values that upper-bound the posterior
variance of the values. The first key insight of our paper is that we can define $u_t$ such that the
$U$-values converge exactly to the variance of values. This result is summarized in the following
theorem:
\begin{restatable}{theorem}{ube}
  \label{thm:ube}
  Under \cref{assumption:transitions,assumption:acyclic}, for any $s \in \mathcal{S}$ and policy
  $\pi$, the posterior variance of the value function, $U_t^\pi = \V_{p \sim \Phi_t}
  \bracket{V^{\pi,p}}$ obeys the uncertainty Bellman equation
  \begin{equation}
    \label{eq:bellman_exact}
    U_t^\pi(s) = 
    \gamma ^ 2u_t(s) + \gamma^2\sum_{a, s'}\pi(a \mid s)\bar{p}_t(s' \mid s,a) U_t^\pi(s'),
  \end{equation}
  where $u_t(s)$ is the local uncertainty defined as
  \begin{equation}
    \label{eq:bellman_exact_reward}
    u_t(s) = \V_{a, s' \sim \pi, \bar{p}_t} \bracket{\bar{V}_t^\pi(s')} -
    \E_{p \sim \Phi_t} \bracket{\V_{a, s' \sim \pi, p} \bracket{V^{\pi, p}(s')}}.
  \end{equation}
\end{restatable}
\begin{proof}
  See \cref{app:proofs_thm_ube}.
\end{proof}
One may interpret the $U$-values from \cref{thm:ube} as the associated state-values of an alternate
\emph{uncertainty MDP}, $\mathcal{U}_t = \set{\mathcal{S}, \mathcal{A}, \bar{p}_t, \rho, \gamma
^2u_t, \gamma ^2}$, where the agent receives uncertainty rewards and transitions according to the
mean dynamics $\bar{p}_t$.

A key difference between $u_t$ and $w_t$ is how they represent epistemic uncertainty: in the former,
it appears only within the first term, through the one-step variance over $\bar{p}_t$; in the
latter, the variance is computed over $\Phi_t$. While the two perspectives may seem fundamentally
different, in the following theorem we present a clear relationship that connects \cref{thm:ube}
with the upper bound \cref{eq:ube_pombu}.
\begin{restatable}{theorem}{connections}
  \label{thm:connection_uncertainties}
  Under \cref{assumption:transitions,assumption:acyclic}, for any $s \in \mathcal{S}$ and policy
  $\pi$, it holds that $u_t(s) = w_t(s) - g_t(s)$, where $g_t(s) = \E_{p \sim
  \Phi_t}\bracket{\V_{a,s' \sim \pi, p} \bracket{V^{\pi,p}(s')}- \V_{a,s' \sim \pi, p}
  \bracket{\bar{V}^\pi_t(s')}}$. Furthermore, we have that the gap $g_t(s)$ is non-negative, thus $u_t(s) \leq w_t(s)$.
\end{restatable}
\begin{proof}
  See \cref{app:proofs_thm_connections}.
\end{proof}

The gap $g_t(s)$ of \cref{thm:connection_uncertainties} can be interpreted as
the \emph{average difference} of aleatoric uncertainty about the next values with respect to the
mean values. The gap vanishes only if the epistemic uncertainty goes to zero, or if the MDP and
policy are both deterministic.

We directly connect \cref{thm:ube,thm:connection_uncertainties} via the equality
\begin{equation}
  \label{eq:interpretation}
	\underbrace{\V_{a, s' \sim \pi, \bar{p}_t} \bracket{\bar{V}_t^\pi(s')}}_\textrm{total} = \underbrace{w_t(s)}_\textrm{epistemic} + \underbrace{\E_{p \sim \Phi_t}\bracket{\V_{a,s' \sim \pi, p} \bracket{\bar{V}^\pi_t(s')}}}_\textrm{aleatoric},
\end{equation}
which helps us analyze our theoretical results. The uncertainty reward defined in
\cref{eq:bellman_exact_reward} has two components: the first term corresponds to the \emph{total
uncertainty} about the \emph{mean} values of the next state, which is further decomposed in
\cref{eq:interpretation} into an epistemic and aleatoric components. When the epistemic uncertainty
about the MDP vanishes, then $w_t(s) \to 0$ and only the aleatoric component remains. Similarly,
when the MDP and policy are both deterministic, the aleatoric uncertainty vanishes and we have
$\V_{a, s' \sim \pi, \bar{p}_t} \bracket{\bar{V}_t^\pi(s')} = w_t(s)$. The second term of
\cref{eq:bellman_exact_reward} is the \emph{average aleatoric uncertainty} about the value of the
next state. When there is no epistemic uncertainty, this term is non-zero and exactly equal to the
alectoric term in \cref{eq:interpretation} which means that $u_t(s) \to 0$. Thus, we can interpret
$u_t(s)$ as a \emph{relative} local uncertainty that subtracts the average aleatoric noise out of
the total uncertainty around the mean values. Perhaps surprisingly, our theory allows negative
$u_t(s)$ (see \cref{subsec:toy_example} for a concrete example).

Through \cref{thm:connection_uncertainties} we provide an alternative proof of why the UBE
\cref{eq:ube_pombu} results in an upper-bound of the variance, specified by the next corollary.
\begin{corollary}
  \label{cor:pombu}
  Under \cref{assumption:transitions,assumption:acyclic}, for any $s \in \mathcal{S}$ and policy
  $\pi$, it holds that the solution to the uncertainty Bellman equation \cref{eq:ube_pombu}
  satisfies $W_t^\pi(s) \geq U_t^\pi(s)$.
\end{corollary}
\begin{proof}
  The solution to the Bellman equations \cref{eq:bellman_exact,eq:ube_pombu} are the value functions
  under some policy $\pi$ of identical MDPs except for their reward functions. Given two identical
  MDPs $\mathcal{M}_1$ and $\mathcal{M}_2$ differing only on their corresponding reward functions
  $r_1$ and $r_2$, if $r_1 \leq r_2$ for any input value, then for any trajectory $\tau$ we have
  that the returns (sum of discounted rewards) must obey $R_1(\tau) \leq R_2(\tau)$. Lastly, since
  the value functions $V_1^\pi$, $V_2^\pi$ are defined as the expected returns under the same
  trajectory distribution, and the expectation operator preserves inequalities, then we
  have that $R_1(\tau) \leq R_2(\tau) \implies  V_1^\pi \leq V_2^\pi$.
\end{proof}
\Cref{cor:pombu} reaches the same conclusions as \citet{zhou_deep_2020}, but it brings important
explanations about their upper bound on the variance of the value function. First, by
\cref{thm:connection_uncertainties} the upper bound is a consequence of the over approximation of
the reward function used to solve the UBE. Second, the gap between the exact reward function
$u_t(s)$ and the approximation $w_t(s)$ is fully characterized by $g_t(s)$ and brings interesting
insights. In particular, the influence of the gap term depends on the stochasticity of the dynamics
and the policy. In the limit, the term vanishes under deterministic transitions and action
selection. In this scenario, the upper-bound found by \citet{zhou_deep_2020} becomes tight.

Our method returns the exact \emph{epistemic} uncertainty about the values by considering the
inherent aleatoric uncertainty of the MDP and the policy. In a practical RL setting, disentangling
the two sources of uncertainty is key for effective exploration. We are interested in exploring
regions of high epistemic uncertainty, where new knowledge can be obtained. If the variance estimate
fuses both sources of uncertainty, then we may be guided to regions of high uncertainty but with
little information to be gained.

\subsection{Toy Example}
\label{subsec:toy_example}
To illustrate the theoretical findings of this paper, consider the simple Markov reward process
(MRP) of \cref{fig:toy_mdp}. Assume $\delta$ and $\beta$ to be random variables drawn from a
discrete uniform distribution $\delta \sim \text{Unif}(\set{0.7, 0.6})$ and $\beta \sim
\text{Unif}(\set{0.5, 0.4})$. As such, the distribution over possible MRPs is finite and composed of
the four possible combinations of $\delta$ and $\beta$. Note that the example satisfies
\cref{assumption:transitions,assumption:acyclic}. In \cref{tab:uncertainty_estimation_toy_example}
we include the results for the uncertainty rewards and solution to the respective UBEs (the results
for $s_1$ and $s_3$ are trivially zero). For state $s_2$, the upper-bound $W^\pi$ is tight and we
have $W^\pi(s_2) = U^\pi(s_2)$. In this case, the gap vanishes not because of lack of stochasticity,
but rather due to lack of epistemic uncertainty about the next-state values. Indeed, the values for
$s_3$ and $s_T$ are independent of $\delta$ and $\beta$, which results in the gap terms for $s_2$
cancelling out. For state $s_0$ the gap is non-zero and $W^\pi$ overestimates the variance of the
value by $\sim36\%$. Our UBE formulation prescribes a \emph{negative} reward to be propagated in
order to obtain the correct posterior variance.

\begin{figure}[t]
  \centering
  \includegraphics[width=0.9\columnwidth]{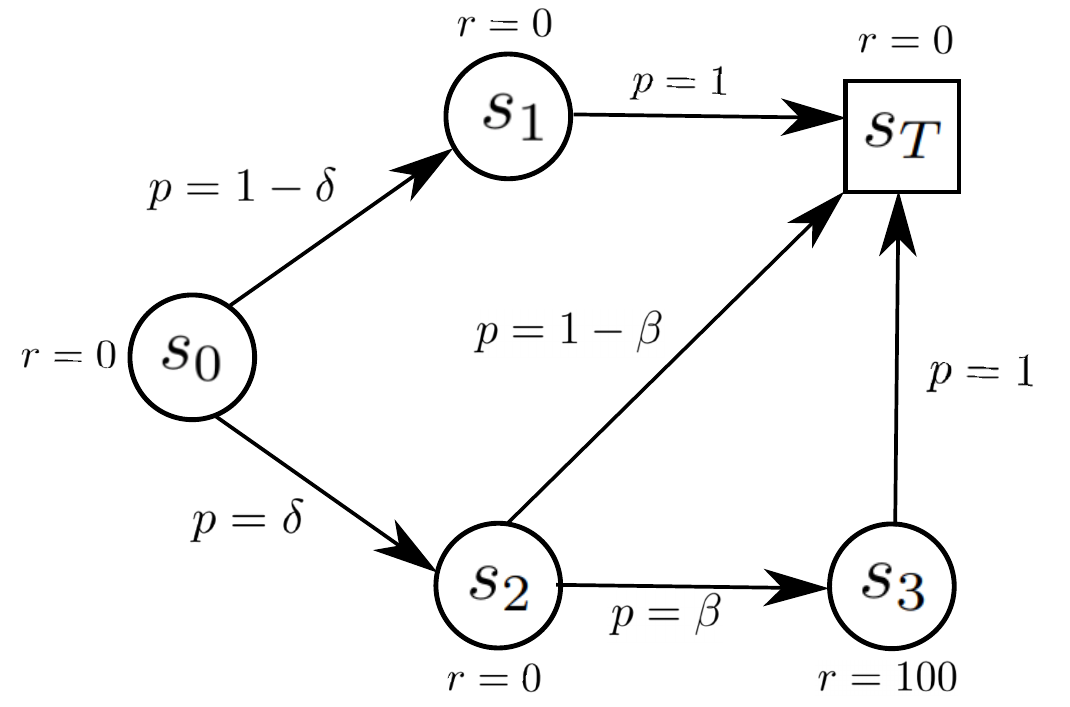}
  \caption{Toy example Markov Reward Process. The random variables $\delta$ and $\beta$ indicate epistemic uncertainty about the MRP. State $s_T$ is an absorbing (terminal) state.}
  \label{fig:toy_mdp} 
\end{figure}

\begin{table}[t]
\caption{Comparison of local uncertainty rewards and solutions to the UBE associated with the toy
example from \Cref{fig:toy_mdp}. The $U$-values converge to the true posterior variance of the values, while $W^\pi$ obtains an upper-bound.} \label{tab:uncertainty_estimation_toy_example}
\begin{center}
\begin{tabular}{c|c|c|c|c}
\textbf{States}  & $u(s)$ & $w(s)$ & $W^\pi(s)$ & $U^\pi(s)$ \\
\hline
$s_0$  & $-0.6$ & $5.0$ & $21.3$ & $15.7$ \\
$s_2$  & $25.0$ & $25.0$ & $25.0$ & $25.0$ \\
\end{tabular}
\end{center}
\end{table}

\section{VARIANCE-DRIVEN OPTIMISTIC EXPLORATION}
\label{sec:optimistic_exploration}
In this section, we propose a technique that leverages uncertainty quantification of $Q$-values to
solve the RL problem. In what follows, we consider the general setting with unknown rewards and
define $\Gamma_t$ to be the posterior distribution over MDPs, from which we can sample both reward
and transition functions. Define $\hat{U}^\pi_t$ to be an estimate of the posterior variance over
$Q$-values for some policy $\pi$ at episode $t$. Then, we update the policy by solving the
upper-confidence bound (UCB) \citep{auer_logarithmic_2006} optimization problem
\begin{equation}
  \label{eq:policy_opt}
  \pi_t = \argmax\nolimits_\pi \bar{Q}^\pi_t + \lambda \sqrt{\hat{U}^\pi_t},
\end{equation}
where $\bar{Q}^\pi_t$ is the posterior mean value function and $\lambda$ is a parameter that trades
off exploration and exploitation. We use \cref{algorithm:our_algorithm} to estimate $\bar{Q}^\pi_t$
and $\hat{U}^\pi_t$: we sample an ensemble of $N$ MDPs from the current posterior $\Gamma_t$ in
\cref{line:sample_model} and use it to solve the Bellman expectation equation in
\cref{line:q_function}, resulting in an ensemble of $N$ corresponding $Q$ functions and the
posterior mean $\bar{Q}^\pi_t$. Lastly, $\hat{U}^\pi_t$ is estimated in \cref{line:q_variance} via a
generic variance estimation method \texttt{qvariance} for which we consider three implementations:
\texttt{ensemble-var} computes a sample-based approximation of the variance given by
$\V\bracket{Q_i}$, which is a model-based version of the estimate from \citet{chen_ucb_2017};
\texttt{pombu} uses the solution to the UBE \cref{eq:ube_pombu}; and \texttt{exact-ube} uses the
solution to our proposed UBE \cref{eq:bellman_exact}. For the UBE-based methods we use the
equivalent equations for $Q$-functions, see \cref{app:q_uncertainty_rewards} for details.

\begin{algorithm}[tb]
   \caption{Model-based $Q$-variance estimation}
   \label{algorithm:our_algorithm}
\begin{algorithmic}[1]
  \STATE {\bfseries Input:} Posterior MDP $\Gamma_t$, policy $\pi$.

  \STATE $\set{p_i, r_i}_{i=1}^{N}$ $\leftarrow$ \texttt{sample\textunderscore mdp}$(\Gamma_{t})$
  \label{line:sample_model}

  \STATE $\bar{Q}^\pi_t$, $\set{Q_i}_{i=1}^{N}$ $\leftarrow$\texttt{solve\textunderscore
  bellman}$\paren{\set{p_i, r_i}_{i=1}^{N}, \pi}$ \label{line:q_function}

  \STATE $\hat{U}^\pi_t$ $\leftarrow$ \texttt{qvariance}$\paren{\set{p_i, r_i, Q_i}_{i=1}^{N}, \bar{Q}^\pi_t, \pi}$ \label{line:q_variance}
\end{algorithmic}
\end{algorithm}

\paragraph{Practical bound.} In practice, typical RL techniques for model learning violate our
theoretical assumptions. For tabular implementations, flat prior choices like a Dirichlet
distribution violate \cref{assumption:acyclic} while function approximation introduces correlations
between states and thus violates \cref{assumption:transitions}. A challenge arises in this practical
setting: \texttt{exact-ube} may result in \emph{negative} $U$-values, as a combination of
(\textit{i}) the assumptions not holding and (\textit{ii}) the possibility of negative uncertainty
rewards. While (\textit{i}) cannot be easily resolved, we propose a practical upper-bound on the
solution of \cref{eq:bellman_exact} such that the resulting $U$-values are non-negative and hence
interpretable as variance estimates. We consider the clipped uncertainty rewards $\tilde{u}_t =
\max(u_{\min}, u_t(s))$ with corresponding $U$-values $\tilde{U}^\pi_t$. It is straightforward to
prove that, if $u_{\min} = 0$, then $W^\pi_t(s) \geq \tilde{U}^\pi_t(s) \geq U^\pi_t(s)$, which
means that using $\tilde{U}^\pi_t$ still results in a tighter upper-bound on the variance than
$W^\pi_t$, while preventing non-positive solutions to the UBE. In what follows, we drop this
notation and assume all $U$-values are computed from clipped uncertainty rewards. Also note that
\texttt{pombu} does not have this problem, since $w_t(s)$ is already non-negative.

\paragraph{Tabular implementation.} For model learning, we impose a Dirichlet prior on the
transition function and a standard Normal prior for the rewards \citep{odonoghue_making_2019}, which
leads to closed-form posterior updates. After sampling $N$ times from the MDP posterior
(\cref{line:sample_model}), we obtain the $Q$-functions (\cref{line:q_function}) in closed-form by
solving the corresponding Bellman equation. For the UBE-based approaches, we estimate uncertainty
rewards via approximations of the expectations/variances therein. Lastly, we solve
\cref{eq:policy_opt} via policy iteration until convergence is achieved or until a maximum number of
steps is reached.

\paragraph{Deep RL implementation.} Inspired by our theory, we propose a deep RL architecture to
scale \Cref{algorithm:our_algorithm} for continuous state-action spaces. Even though there is no
formal proof of the existence of the UBE in this setting, we argue that approximating the sum of
cumulative uncertainty rewards allows for uncertainty propagation.

We adopt as a baseline architecture MBPO by \citet{janner_when_2019} and the implementation from
\citet{pineda_mbrl-lib_2021}. In contrast to the tabular implementation, maintaining an explicit
distribution over MDPs from which we can sample is intractable. Instead, we consider
$\Gamma_t$ to be a discrete uniform distribution of $N$ probabilistic neural networks, denoted
$p_\theta$, that output the mean and covariance of a Gaussian distribution over next states and
rewards \citep{chua_deep_2018}. In this case, the output of \cref{line:sample_model} in
\cref{algorithm:our_algorithm} is precisely the ensemble of neural networks.

The original MBPO trains $Q$-functions represented as neural networks via TD-learning on data
generated via \emph{model-randomized} $k$-step rollouts from initial states that are sampled from
$\mathcal{D}_t$. Each forward prediction of the rollout comes from a randomly selected model of the
ensemble and the transitions are stored in a single replay buffer $\mathcal{D}_{\text{model}}$,
which is then fed into a model-free optimizer like soft actor-critic (SAC)
\citep{haarnoja_soft_2018}. SAC trains a stochastic policy represented as a neural network with
parameters $\phi$, denoted by $\pi_\phi$. The policy's objective function is similar to
\cref{eq:policy_opt} but with entropy regularization instead of the uncertainty term. In practice,
the argmax is replaced by $G$ steps of stochastic gradient ascent, where the policy gradient is
estimated via mini-batches drawn from $\mathcal{D}_{\text{model}}$.

\Cref{algorithm:our_algorithm} requires a few modifications from the MBPO methodology. To implement
\cref{line:q_function}, in addition to $\mathcal{D}_{\text{model}}$, we create $N$ new buffers
$\set{\mathcal{D}^i_{\text{model}}}_{i=1}^{N}$ filled with \emph{model-consistent} rollouts, where
each $k$-step rollout is generated under a single model of the ensemble, starting from initial
states sampled from $\mathcal{D}_t$. We train an ensemble of $N$ value functions
$\set{Q_i}_{i=1}^{N}$, parameterized by $\set{\psi_i}_{i=1}^{N}$, and minimize the residual
Bellman error with entropy regularization 
\begin{equation}
  \label{eq:loss_q}
  \mathcal{L}(\psi_i) = \E_{(s,a,r,s') \sim \mathcal{D}_t^i} \bracket{\paren{y_i- Q_i(s, a; \psi_i))}^2}, 
\end{equation}
where $y_i = r + \gamma \paren{Q_i(s', a'; \bar{\psi}_i) - \alpha \log \pi_\phi(a' \mid s')}$
and $\bar{\psi}_i$ are the target network parameters updated via Polyak averaging for
stability during training \citep{mnih_playing_2013}. The mean $Q$-values, $\bar{Q}^\pi_t$, are
estimated as the average value of the $Q$-ensemble. 

\begin{figure*}
	\centering
  \includegraphics{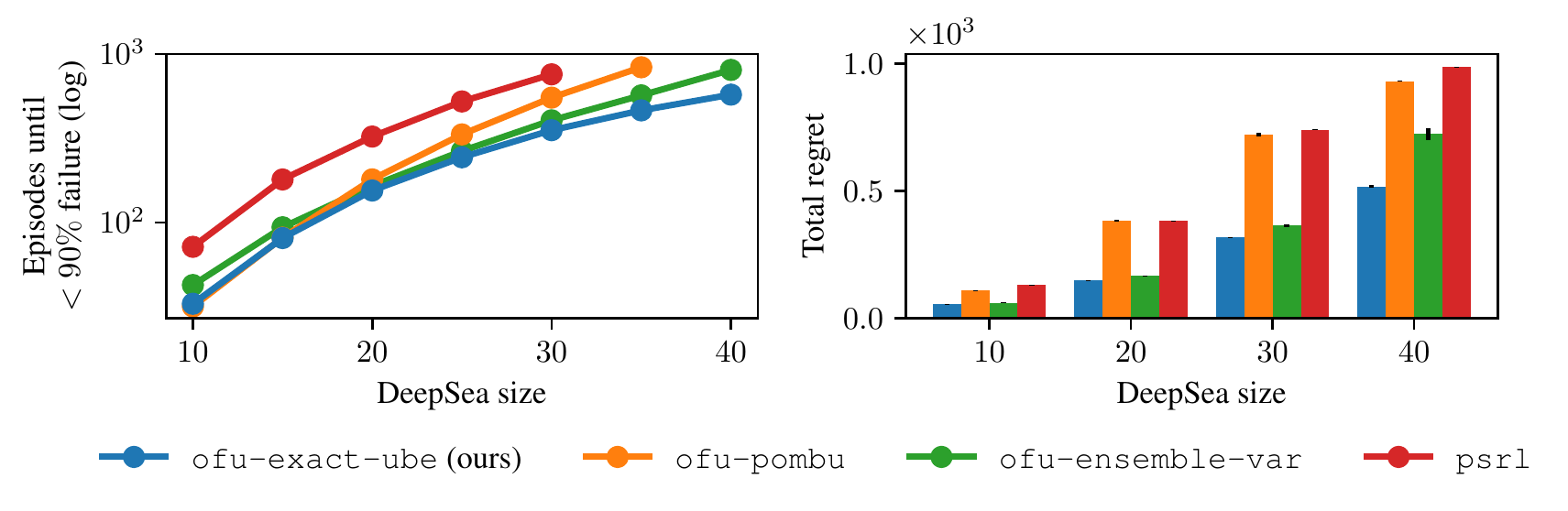}
  \caption{Performance in the \emph{DeepSea} benchmark. Lower values in plots indicate better
  performance. (Left) Learning time is measured as the first episode where the sparse reward has
  been found at least in 10\% of episodes so far. (Right) Total regret is approximately equal to the
  number of episodes where the sparse reward was not found. Results represent the average over 5
  random seeds, and vertical bars on total regret indicate the standard error. Our variance estimate
  achieves the lowest regret and best scaling with problem size.}
  \label{fig:deep_sea_results}
\end{figure*}

To approximate the solution to the UBE, we train a neural network parameterized by a vector
$\varphi$, denoted $U_\varphi$ (informally, the $U$-net). Since we interpret the output of the
network as predictive variances, we (i) regularize the output to be positive by penalizing
negative values and (ii) use a \emph{softplus} output layer to guarantee non-negative values. For
regularization, let $f_\varphi$ be the network output before the softplus operation, then we define
the regulatization loss
\begin{equation}
  \mathcal{L}_{\text{reg}}(\varphi) = \E_{(s,a,r,s') \sim
  \mathcal{D_{\text{model}}}} \bracket{\paren{\text{ReLU}(-f_\varphi(s,a) - \epsilon)}^2},
\end{equation}
such that $\mathcal{L}_{\text{reg}}(\varphi) \geq 0$ iff $f_\varphi(s,a) < \epsilon$ for some small
$\epsilon > 0$. Otherwise, for $f_\varphi(s,a) > \epsilon$ the loss is zero and regularization is
inactive. In practice, we found that regularization is key to avoid network collapse in sparse
reward problems, while it is typically not required if rewards are dense. Training of the $U$-net is
carried out by minimizing the uncertainty Bellman error with regularization:
\begin{equation}
  \label{eq:loss_u}
  \mathcal{L}(\varphi) = \E_{(s,a,r,s') \sim \mathcal{D_{\text{model}}}} \bracket{\paren{z - U(s, a; \varphi)}^2} + \lambda_{\text{reg}}\mathcal{L}_{\text{reg}}(\varphi),
\end{equation}
with targets $z = \gamma^2 u(s,a) + \gamma ^2 U(s', a'; \bar{\varphi})$ and target parameters
$\bar{\varphi}$ updated like in regular critics. Lastly, we optimize $\pi_\phi$ as in MBPO via SGD
on the SAC policy loss, but also adding the uncertainty term from \cref{eq:policy_opt}. A
detailed algorithm of our approach is included in \Cref{app:deep_rl_implementation}.

\paragraph{Runtime complexity.} In tabular RL, \texttt{exact-ube} solves $N+2$ Bellman equations
($\bar{Q}^\pi_t$, $Q_i$, $\hat{U}^\pi_t$), \texttt{pombu} solves two ($\bar{Q}^\pi_t$,
$\hat{U}^\pi_t$) and \texttt{ensemble-var} solves $N+1$ ($\bar{Q}^\pi_t$, $Q_i$). In deep RL,
UBE-based methods have the added complexity of training the $U$-net, but it can be parallelized with
the $Q$-ensemble traning. Despite the increased complexity, we show in
\cref{subsec:ensemble_size_ablation} that our method performs well for small $N$, which reduces the
computational burden.

\section{EXPERIMENTS}
\label{sec:experiments}
In this section, we empirically evaluate the performance of the policy optimization scheme
\cref{eq:policy_opt} for the different variance estimates that we introduced in
\cref{sec:optimistic_exploration}. 

\subsection{Tabular Environments}
We evaluate the tabular implementation in grid-world environments. We include PSRL by
\citet{osband_more_2013} as a baseline since it typically outperforms recent OFU-based
methods \citep{odonoghue_variational_2021,tiapkin_dirichlet_2022}.

\paragraph{DeepSea.} First proposed by \citet{osband_deep_2019}, this environment tests the agent's
ability to explore over multiple time steps in the presence of a deterrent. It consists of an $L
\times L$ grid-world MDP, where the agent starts at the top-left cell and must reach the lower-right
cell. The agent decides to move left or right, while always descending to the row below. We consider
the deterministic version of the problem, so the agent always transitions according to the chosen
action. Going left yields no reward, while going right incurs an action cost (negative reward) of
$0.01 / L$. The bottom-right cell yields a reward of 1, so that the optimal policy is to always go
right. As the size of the environment increases, the agent must perform sustained exploration in
order to reach the sparse reward. Detailed implementation and hyperparameter details are included in
\cref{app:experimental_details}. 

The experiment consists on running each method for 1000 episodes and five random seeds, recording
the total regret and ``learning time'', defined as the first episode where the rewarding state has
been found at least in 10\% of the episodes so far \citep{odonoghue_variational_2021}. For this
experiment, we found that using $u_{\min} = -0.05$ improves the performance of our method: since the
underlying MDP is acyclic, propagating negative uncertainty rewards is consistent with our theory.

\begin{figure}[t]
	\centering
  \includegraphics{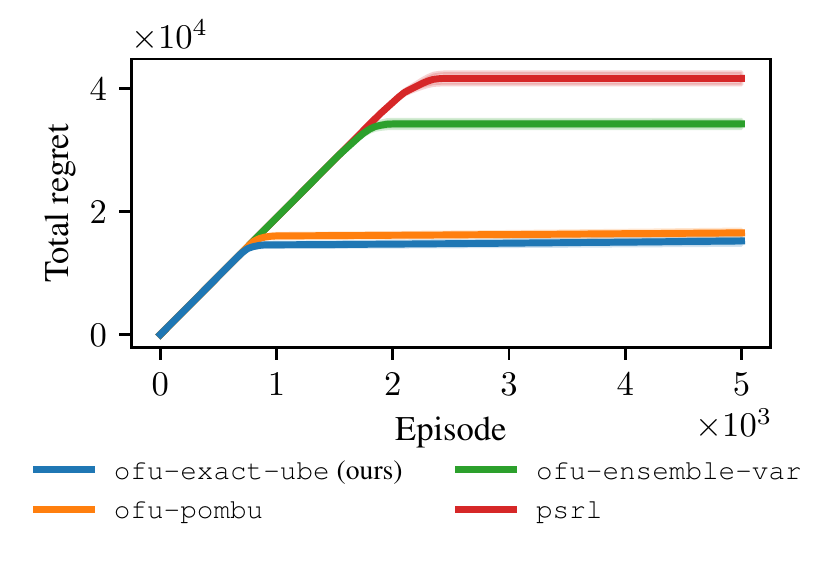}
  \caption{Total regret curve for the 7-room environment. Lower regret is better. Results are the average (solid lines) and
  standard error (shaded regions) over 10 random seeds. Our method achieves the lowest regret,
  significantly outperforming PSRL.}
  \label{fig:nroom}
\end{figure}

\cref{fig:deep_sea_results} (left) shows the evolution of learning time as $L$ increases. Our method
achieves the lowest learning time and best scaling with problem size. Notably, all the OFU-based
methods learn faster than PSRL, a strong argument in favour of using the variance of value functions
to guide exploration. \Cref{fig:deep_sea_results} (right) shows that our approach consistently
achieves the lowest total regret across all values of $L$. This empirical evidence indicates that
the solution to our UBE can be integrated into common exploration techniques like UCB to serve as an
effective uncertainty signal. Moreover, our method significantly improves peformance over
\texttt{pombu}, highlighting the relevance of our theory results.

\begin{figure*}[ht]
	\centering
  \includegraphics{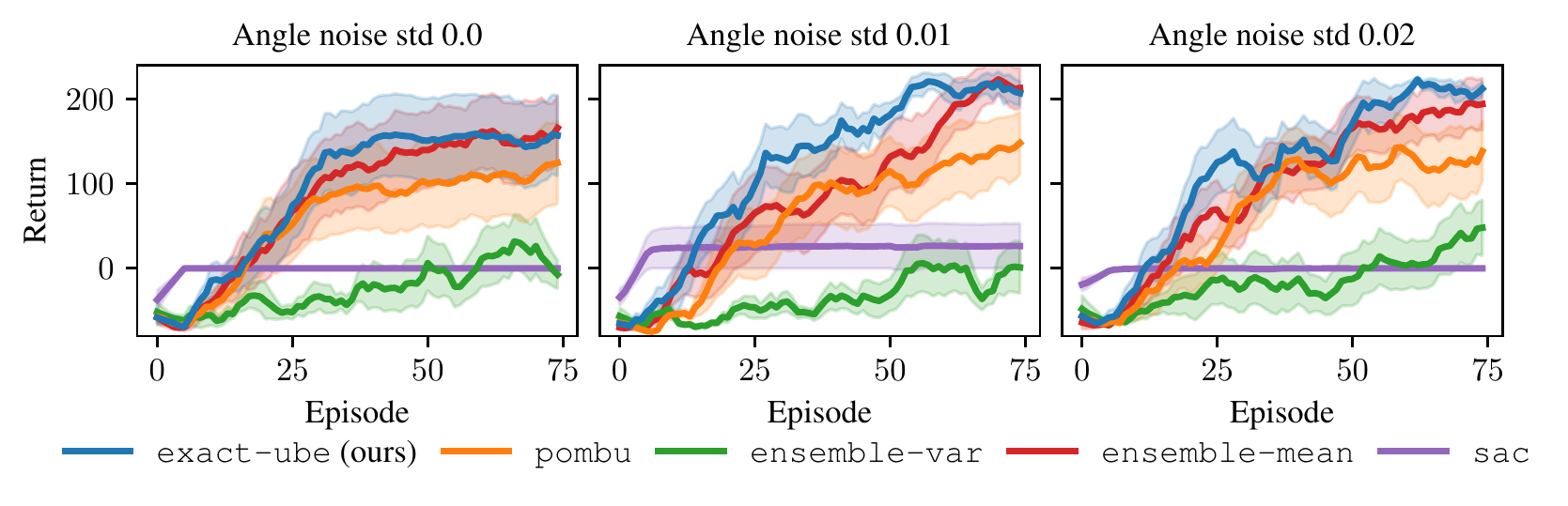}
  \caption{Learning curves of the pendulum swing-up with sparse rewards and action costs. Gaussian
  noise of different scales is added to the angle of the pendulum. The returns are smoothened by a
  moving average filter and we report the mean (solid lines) and standard error (shaded regions)
  over 10 random seeds. Our method shows some improvement in sample efficiency and comparable or
  higher final performance than the baselines.}
  \label{fig:pendulum_returns}
\end{figure*}

Detailed results of all the runs are included in \cref{app:extended_results}. Additional ablation
studies on different estimates for our UBE and exploration gain $\lambda$ are included in
\cref{app:ablation_exact_ube,app:ablation_exploration_gain}, respectively.

\paragraph{7-room.} As implemented by \citet{domingues_rlberry_2021}, the 7-room environment
consists of seven connected rooms of size $5\times 5$. The agent starts in the center of the middle
room and an episode lasts 40 steps. The possible actions are up-down-left-right and the agent
transitions according to the selected action with probability $0.95$, otherwise it lands in a random
neighboring cell. The environment has zero reward everywhere except two small rewards at the start
position and in the left-most room, and one large reward in the right-most room. Unlike
\emph{DeepSea}, the underlying MDP for this environment contains cycles, so it evaluates our method
beyond the theoretical assumptions. In \cref{fig:nroom}, we show the regret curves over 5000
episodes. Our method achieves the lowest regret, which is remarkable considering recent empirical
evidence favoring PSRL over OFU-based methods in these type of environments
\citep{tiapkin_dirichlet_2022}. The large gap between \texttt{ensemble-var} and the UBE-based
methods is due to overall larger variance estimates from the former, which consequently requires
more episodes to reduce the value uncertainty.

\subsection{Continuous Control Environments}
\label{subsec:continuous_control_experiments}
In this section, we evaluate the performance of the deep RL implementation in environments with
continuous state-action spaces. Implementation details and hyperparameters are included in
\Cref{app:deep_rl_implementation}.

\paragraph{Sparse Inverted Pendulum.} As proposed by \citet{curi_efficient_2020}, non-zero rewards
only exist close to the upward position. The pendulum is always initialized in the downward position
with zero velocity and one episode lasts 400 steps. Stochasticity is introduced via zero-mean
Gaussian noise in the pendulum's angle. We complicate the problem further by adding an action cost,
which directly counteracts the effect of exploration signals. The combination of sparse rewards and
action costs represent a failure case for model-free approaches relying on the stochasticity of the
policy to explore (e.g. SAC). While noisy transitions may actually \emph{help} solve these problems
by increasing the random chance of encountering the sparse rewards, they also motivate the need for
proper filtering of aleatoric noise when estimating the epistemic uncertainty.

\begin{figure*}[ht]
	\centering
  \includegraphics{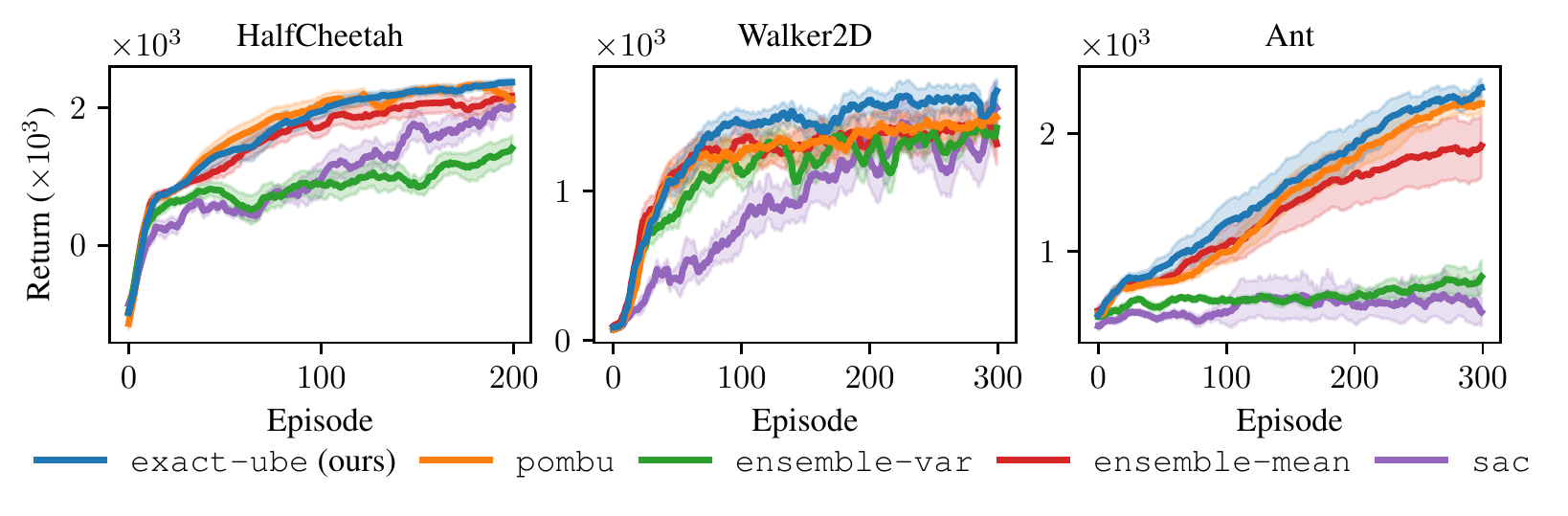}
  \caption{Learning curves in Pybullet locomotion environments. Returns are smoothened by a moving
  average and we report the mean (solid lines) and standard error (shaded regions) over 10 random
  seeds. While these environments have dense rewards, our UBE-based exploration method shows
  improvements in terms of learning speed and final performance.}
  \label{fig:pybullet_results}
\end{figure*}

\begin{figure*}[ht]
	\centering
  \includegraphics{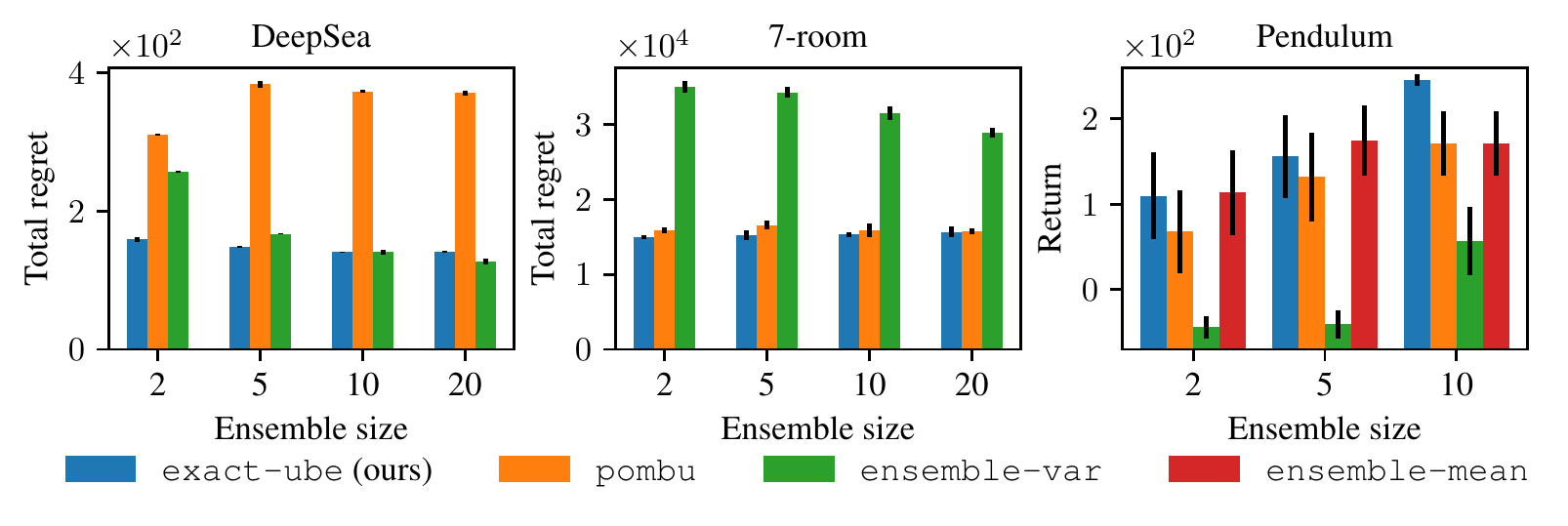}
  \caption{Ablation study for the ensemble size $N$. We report the mean/standard error of the final
  total regret for \emph{DeepSea} ($L=30$) and 7-room across five and ten seeds, respectively. For
  the sparse pendulum, we set the angle noise standard deviation to $0$ and show the mean/standard
  error of the final return after 75 episodes across ten seeds. All methods improve performance for
  larger $N$, but our method is able to achieve the best overal performance.}
  \label{fig:ensemble_ablation}
\end{figure*}

The benchmark includes two additional baselines: \texttt{ensemble-mean}, which uses no optimism and
only averages over the epistemic uncertainty of the $Q$-ensemble, and SAC.
\Cref{fig:pendulum_returns} shows the learning curves over 75 episodes for three different noise
levels. SAC quickly converges to the suboptimal solution of not applying any torque to the pendulum,
while all model-based approaches avoid this pitfall. Overall, our \texttt{exact-ube} method has the
most robust performance across the different noise levels, in most cases improving sample-efficiency
and achieving comparable or higher final return. Importantly, \texttt{exact-ube} outperforms
\texttt{pombu} in all scenarios, which is consistent with our theoretical insights about our method
better handling aleatoric uncertainty. Perhaps surprisingly, greedily averaging over the epistemic
uncertainty (\texttt{ensemble-mean}) is a strong baseline. Meanwhile, the \texttt{ensemble-var}
method tends to over-explore due to higher variance estimates than the UBE-based methods, leading to
more erratic learning curves and lower sample-efficiency (see \Cref{app:pendulum_viz} for a
visualization). 

\paragraph{PyBullet Locomotion.} We evaluate performance on three locomotion tasks from the PyBullet
suite \citep{coumans_pybullet_2016}, which have increased dimensionality compared to the simple
pendulum environment. Although these environments have dense rewards, thus arguably less need for
deep exploration, the results in \cref{fig:pybullet_results} demonstrate some performance
improvement using \texttt{exact-ube} compared to the baselines. Similar to the pendulum task,
\texttt{ensemble-var} affords higher variance estimates which severely hinders performance, while
\texttt{ensemble-mean} is a strong baseline upon which some improvements can be afforded with
UBE-based optimism.

While we cannot make broad claims based on these results, they provide supporting evidence that: (1)
UBE-based methods can be scaled to continuous-control problems using $U$-nets and (2) our UBE
formulation provides benefits in solving RL tasks with respect to prior work.

\subsection{Ensemble Size Ablation}
\label{subsec:ensemble_size_ablation}
The ensemble size $N$ represents a critical hyperparameter for ensemble-based methods, balancing
compute and sample diversity. The work by \citet{an_uncertainty-based_2021} suggests that classical
ensemble methods may require large $N$ to achieve good performance, which is computationally
expensive. We evaluate this hypothesis through an ablation study over $N$ across different
exploration tasks. The results in \cref{fig:ensemble_ablation} show that our method achieves the
best or comparable performance across all environments and values of $N$. The \texttt{ensemble-var}
estimate is more sensitive to $N$ and its performance increases for larger ensembles, matching the
observations from \citet{an_uncertainty-based_2021}. We hypothesize that sample-based approximations
of the local uncertainty rewards, which typically have small magnitude, are less sensitive to sample
size than directly estimating variance from the ensemble members. Further experiments in the
pendulum environment (included in \Cref{app:sparse_pendulum_ensemble_ablation}) suggest that larger
ensembles may not always lead to better performance in the presence of sparse rewards; in the
absence of a strong reward signal, most ensemble members will agree on predicting close-to-zero
values which may then lead to premature convergence of the policy. We hypothesize that for larger
ensembles it is key to promote sufficient diversity to avoid variance collapse and solve the task.

\section{CONCLUSIONS}
In this paper, we derived an uncertainty Bellman equation whose fixed-point solution converges to
the variance of values given a posterior distribution over MDPs. Our theory brings new understanding
by characterizing the gap in previous UBE formulations that upper-bound the variance of values. We
showed that this gap is the consequence of an over-approximation of the uncertainty rewards being
propagated through the Bellman recursion, which ignore the inherent \emph{aleatoric} uncertainty
from acting in an MDP. Instead, our theory recovers exclusively the \emph{epistemic} uncertainty due
to limited environment data, thus serving as an effective exploration signal.

We proposed a practical method to estimate the solution of the UBE, scalable beyond tabular problems
with standard deep RL practices. Our variance estimation was integrated into a model-based approach
using the principle of optimism in the face of uncertainty to explore effectively. Experimental
results showed that our method improves sample efficiency in hard exploration problems and without
requiring large ensembles.

\typeout{}
\bibliographystyle{plainnatnourl}
\bibliography{references}


\clearpage
\appendix

\thispagestyle{empty}

\onecolumn
\part[Supplementary Material]{}
\aistatstitle{Supplementary Material: \\
Model-Based Uncertainty in Value Functions}

\mtcsetdepth{parttoc}{4}

\noptcrule
\vspace{-30ex}
\parttoc 

\clearpage

\section{THEORY PROOFS}
\label{app:proofs}
\subsection{Proof of \cref{thm:ube}}
\label{app:proofs_thm_ube}
In this section, we provide the formal proof of \cref{thm:ube}. We begin by showing an expression
for the posterior variance of the value function without assumptions on the MDP. We define the joint
distribution $p^\pi(a, s' \mid s) = \pi(a\mid s)p(s' \mid s,a)$ for a generic transition function
$p$. To ease notation, since $\pi$ is fixed, we will simply denote the joint distribution as $p(a,s'
\mid s)$. 

\begin{lemma}
  \label{lemma:variance_decomp_no_assumptions}
  For any $s \in \mathcal{S}$ and any policy $\pi$, it holds that
  \begin{equation}
    \label{eq:variance_decomposition}
    \V_{p \sim \Phi_t} \bracket{V^{\pi, p}(s)} = \gamma^2\E_{p \sim \Phi_t} \bracket{
      \paren{
        \sum_{a,s'}p(a, s' \mid s) V^{\pi, p}(s')
      }^2
    } - \gamma^2
    \paren{
      \E_{p \sim \Phi_t} \bracket{
        \sum_{a,s'}p(a,s' \mid s) V^{\pi, p}(s')
      }
    }^2.
  \end{equation}
\end{lemma}
\begin{proof}
  Using the Bellman expectation equation
  \begin{equation}
    \label{eq:bellman_expectation}
    V^{\pi,p}(s) = \sum_{a} \pi(a \mid s)r(s,a) + \gamma \sum_{a, s'}p(a,s' \mid s)V^{\pi, p}(s'),
  \end{equation}
  we have
  \begin{align}
    \V_{p \sim \Phi_t} \bracket{V^{\pi, p}(s)} &= \V_{p \sim \Phi_t} \bracket{\sum_a \pi(a \mid s)r(s,a) + \gamma\sum_{a, s'} p(a,s' \mid s)V^{\pi, p}(s')} \\
    &= \V_{p \sim \Phi_t} \bracket{\gamma\sum_{a,s'}p(a, s' \mid s)V^{\pi, p}(s')},  \label{eq:var_no_assumption}
  \end{align}
  where \cref{eq:var_no_assumption} holds since $r(s,a)$ is deterministic. Using the identity $\V[Y]
  = \E[Y^2] - (\E[Y])^2$ on \cref{eq:var_no_assumption} concludes the proof. 
\end{proof}

The next result is the direct consequence of our set of assumptions.
\begin{lemma}
  \label{lemma:independence_from_assumptions}
  Under \cref{assumption:transitions,assumption:acyclic}, for any $s \in \mathcal{S}$, any policy
  $\pi$, $\Cov[p(s' \mid s,a), V^{\pi,p}(s')] = 0$.
\end{lemma}
\begin{proof}
  Define $\tau$ to be any trajectory starting from state $s'$, $\tau = \set{s', a_0, s_1, a_1,
  \dots}$. First, by \cref{assumption:transitions}, if $s_i \neq s'$ for some $i \in \set{1, 2,
  \dots}$, then $p(s' \mid s,a)$ is independent of $p(s' \mid s_i, a)$. However, by
  \cref{assumption:acyclic}, $s_i \neq s'$ for all $i>0$, which implies that the trajectory
  distribution $P(\tau)$ is independent of the transition $p(s' \mid s,a)$. Lastly, since
  $V^{\pi,p}(s')$ is an expectation under $P(\tau)$, and independence implies zero correlation, the
  lemma holds.
\end{proof}

Using the previous result yields the following lemma.
\begin{lemma}
  \label{lemma:uncorrelated_property}
  Under \cref{assumption:transitions,assumption:acyclic}, it holds that
  \begin{equation}
    \label{eq:uncorrelated_property}
    \sum_{a, s'} \E_{p \sim \Phi_t} \bracket{
      p(a, s' \mid s) V^{\pi, p}(s')
    } = 
    \sum_{a, s'} \bar{p}_t(a, s' \mid s) \E_{p \sim \Phi_t} \bracket{V^{\pi, p}(s')}.
  \end{equation}
\end{lemma}
\begin{proof}
  For any pair of random variables $X$ and $Y$ on the same probability space, by definition of
  covariance it holds that $\E[XY] = \Cov[X,Y] + \E[X]\E[Y]$. Using this identity with
  \cref{lemma:independence_from_assumptions} and the definition of posterior mean transition
  \cref{eq:mean_model} yields the result.
\end{proof}

Now we are ready to prove the main theorem.
\ube*

\begin{proof}
  Starting from the result in \cref{lemma:variance_decomp_no_assumptions}, we consider each term on
  the r.h.s of \cref{eq:variance_decomposition} separately. For the first term, notice that within
  the expectation we have a squared expectation over the transition probability $p(s' \mid s,a)$,
  thus using the identity $(\E[Y])^2 = \E[Y^2] - \V[Y]$ results in
  \begin{align}
    \E_{p \sim \Phi_t} \bracket{
      \paren{
        \sum_{a,s'}p(a,s' \mid s) V^{\pi, p}(s')
      }^2
    } &= 
    \E_{p \sim \Phi_t} \bracket{
      \sum_{a,s'}p(a,s' \mid s)\paren{V^{\pi,p}(s')}^2 -
      \V_{a,s' \sim \pi, p} \bracket{V^{\pi,p}(s')}
    }.  \intertext{Applying linearity of expectation to bring it inside the sum
    and an application of \cref{lemma:uncorrelated_property} (note that the lemma applies for squared values as well) gives}
    &= 
    \label{eq:thm1_first_term}
    \sum_{a,s'}\bar{p}_t(a,s' \mid s)\E_{p \sim \Phi_t}\bracket{\paren{V^{\pi,p}(s')}^2} - \E_{p \sim \Phi_t}\bracket{\V_{a,s' \sim \pi,p} \bracket{V^{\pi,p}(s')}
    }.
  \end{align}
  For the second term of the r.h.s of \cref{eq:variance_decomposition} we apply again
  \cref{lemma:uncorrelated_property} and under definition of variance
  \begin{align}
    \paren{
      \E_{p \sim \Phi_t} \bracket{
        \sum_{a,s'}p(a,s' \mid s) V^{\pi, p}(s')
      }
    }^2 \label{eq:prev_thm1_second_term}
    &= \paren{
      \sum_{a,s'}\bar{p}_t(a,s' \mid s)\E_{p \sim \Phi_t}\bracket{V^{\pi,p}(s')}
    }^2 \\
    &= \label{eq:thm1_second_term} \sum_{a,s'}\bar{p}_t(a,s' \mid s)\paren{
      \E_{p \sim \Phi_t}\bracket{V^{\pi,p}(s')}
    }^2 - \V_{a,s' \sim\pi, \bar{p}_t}\bracket{\E_{p \sim \Phi_t}\bracket{V^{\pi,p}(s')}}.
  \end{align}
  Finally, since
  \begin{equation}
    \E_{p \sim \Phi_t}\bracket{\paren{V^{\pi,p}(s')}^2} -
    \paren{
      \E_{p \sim \Phi_t}\bracket{V^{\pi,p}(s')}
    }^2  = \V_{p \sim \Phi_t}\bracket{V^{\pi,p}(s')}
  \end{equation}
  for any $s' \in \mathcal{S}$, we can plug \cref{eq:thm1_first_term,eq:thm1_second_term} into
  \cref{eq:variance_decomposition}, which proves the theorem.
\end{proof}

\subsection{Proof of \cref{thm:connection_uncertainties}}
\label{app:proofs_thm_connections}
In this section, we provide the supporting theory and the proof of
\cref{thm:connection_uncertainties}. First, we will use the identity $\V[\E[Y|X]] = \E[(\E[Y|X])^2]
- (\E[E[Y|X]])^2$ to prove $u_t(s) = w_t(s) - g_t(s)$ holds, with ${Y = \sum_{a,s'}p(a,s' \mid
s)V^{\pi,p}(s')}$. For the conditioning variable $X$, we define a transition function with fixed
input state $s$ as a mapping $p_s : \mathcal{A} \to \Delta(S)$ representing a distribution $p_s(s'
\mid a) = p(s' \mid s,a)$. Then $X = \mathbf{P}_s := \set{p_s(s' \mid a)}_{s' \in \mathcal{S}, a \in
\mathcal{A}}$. The transition function $p_s$ is drawn from a distribution $\Phi_{s,t}$ obtained by
marginalizing $\Phi_t$ on all transitions not starting from $s$. 
\begin{lemma}
  \label{lemma:pombu_local_uncertainty}
  Under \cref{assumption:transitions,assumption:acyclic}, it holds that 
  \begin{equation}
    \V_{p_s \sim \Phi_{s,t}} \bracket{\E_{p \sim \Phi_t} \bracket{\sum_{a,s'} p(a,s' \mid s)V^{\pi,p}(s') \given \mathbf{P}_s}} = \V_{p \sim \Phi_t} \bracket{\sum_{a,s'} p(a,s' \mid s)\bar{V}^\pi_t(s')}.
  \end{equation}
\end{lemma}
\begin{proof}
  Treating the inner expectation,
  \begin{align}
    \E_{p \sim \Phi_t} \bracket{\sum_{a,s'} p(a,s' \mid s)V^{\pi,p}(s')\mid \textbf{P}_s} &= \sum_{a}\pi(a \mid s) \sum_{s'}\E_{p \sim \Phi_t}\bracket{p(s' \mid s,a)V^{\pi, p}(s') \given \mathbf{P}_s}. 
    \intertext{Due to the conditioning, $p(s' \mid s,a)$ is deterministic within the expectation}
    &= \sum_{a,s'} p(a,s' \mid s) \E_{p \sim \Phi_t}\bracket{V^{\pi,p}(s') \given \mathbf{P}_s}.
    \intertext{By \cref{lemma:independence_from_assumptions}, $V^{\pi,p}(s')$ is independent of $\mathbf{P}_s$, so we can drop the conditioning}
    &= \sum_{a,s'} p(a,s' \mid s)\bar{V}^\pi_t(s').
  \end{align}
  Lastly, since drawing samples from a marginal distribution is equivalent to drawing samples from
  the joint, i.e., $\V_x[f(x)] = \V_{(x,y)}[f(x)]$, then:
  \begin{equation}
    \V_{p_s \sim \Phi_{s,t}}\bracket{\sum_{a,s'} p(a,s' \mid s)\bar{V}^\pi_t(s')} = \V_{p \sim \Phi_t}\bracket{\sum_{a,s'} p(a,s' \mid s)\bar{V}^\pi_t(s')},
  \end{equation}
  completing the proof.
\end{proof}

The next lemma establishes the result for the expression $\E[(\E[Y|X])^2]$.
\begin{lemma}
  \label{lemma:rhs_pombu_uncertainty_first}
  Under \cref{assumption:transitions,assumption:acyclic}, it holds that 
  \begin{align}
    \E_{p_s \sim \Phi_{s,t}} \bracket{\paren{\E_{p \sim \Phi_t} \bracket{\sum_{a,s'}p(a,s' \mid s)V^{\pi,p}(s')\given \mathbf{P}_s}}^2} &= \sum_{a,s'}\bar{p}_t(a,s' \mid s)\paren{\bar{V}^\pi_t(s')} - \E_{p \sim \Phi_t}\bracket{\V_{a,s' \sim \pi,p} \bracket{\bar{V}^\pi_t(s')}}.
  \end{align}
  \begin{proof}
    The inner expectation is equal to the one in \cref{lemma:pombu_local_uncertainty}, so we have
    that 
    \begin{align}
      \paren{\E_{p \sim \Phi_t} \bracket{\sum_{a,s'}p(a,s' \mid s)V^{\pi,p}(s')\given \mathbf{P}_s}}^2 &= \paren{\sum_{a,s'}p(a,s' \mid s) \bar{V}^\pi_t(s')}^2 \\
      &= \sum_{a,s'}p(a,s' \mid s)\paren{\bar{V}^\pi_t(s')}^2 - \V_{a,s' \sim \pi,p}\bracket{\bar{V}^\pi_t(s')}. \label{eq:term_no_exp}
    \end{align}
    Finally, applying expectation on both sides of \cref{eq:term_no_exp} yields the result.
  \end{proof}
\end{lemma}
Similarly, the next lemma establishes the result for the expression $(\E[\E[Y|X]])^2$.
\begin{lemma}
  \label{lemma:rhs_pombu_uncertainty_second}
  Under \cref{assumption:transitions,assumption:acyclic}, it holds that 
  \begin{align}
    \paren{\E_{p_s \sim \Phi_{s,t}} \bracket{\E_{p \sim \Phi_t} \bracket{\sum_{a,s'}p(a,s' \mid s)V^{\pi,p}(s')\given \mathbf{P}_s}}^2} &= \sum_{a,s'}\bar{p}_t(a,s' \mid s)\paren{\bar{V}^\pi_t(s')} - \V_{a,s' \sim \pi,\bar{p}_t} \bracket{\bar{V}^\pi_t(s')}.
  \end{align}
  \begin{proof}
    By the tower property of expectations, $(\E[\E[Y|X]])^2 = (\E[Y])^2$. Then, the result follows
    directly from \cref{eq:prev_thm1_second_term,eq:thm1_second_term}.
  \end{proof}
\end{lemma}
The second part of \cref{thm:connection_uncertainties} is a corollary of the next lemma.
\begin{lemma}
  \label{lemma:inflated_uncertainty}
  Under \cref{assumption:transitions,assumption:acyclic}, it holds that
  \begin{equation}
    \label{eq:non_negative_quantity}
    \E_{p \sim \Phi_t}\bracket{\V_{a, s' \sim \pi, p}\bracket{V^{\pi,p}(s')} - \V_{a,s' \sim \pi, p}\bracket{\bar{V}^\pi_t(s')}}
  \end{equation}
  is non-negative.
\end{lemma}
\begin{proof}
  We will prove the lemma by showing \cref{eq:non_negative_quantity} is equal to $\E_{p \sim
  \Phi_t}\bracket{\V_{a,s' \sim \pi,p}\bracket{V^{\pi,p}(s') - \bar{V}^\pi_t(s')}}$, which is a
  non-negative quantiy by definition of variance. The idea is to derive two expressions for
  $\E[\V[Y|X]]$ and compare them. First, we will use the identity $\E[\V[Y|X]] = \E[\E[(Y - \E[Y|X])^2
  | X]]$. The outer expectation is w.r.t the marginal distribution $\Phi_{s,t}$ while the inner
  expectations are w.r.t $\Phi_t$. For the inner expectation we have
  \begin{align}
    &
    \E_{p \sim \Phi_t} \bracket{\paren{
      \sum_{a,s'}p(a,s' \mid s)V^{\pi, p}(s') - \E_{p \sim \Phi_t}\bracket{\sum_{a,s'}p(a,s' \mid s)V^{\pi,p}(s') \given \mathbf{P}_s}
    }^2\given \mathbf{P}_s} \\
    &= \E_{p \sim \Phi_t} \bracket{\paren{
      \sum_{a,s'}p(a,s' \mid s) \paren{
        V^{\pi, p}(s') - \E_{p\sim \Phi_t}\bracket{V^{\pi,p}\given \mathbf{P}_s}
      }}^2 \given \mathbf{P}_s} \\
    &= \E_{p \sim \Phi_t} \bracket{\paren{
      \sum_{a,s'}p(a,s' \mid s) \paren{
        V^{\pi, p}(s') - \bar{V}^\pi_t(s')
      }}^2 \given \mathbf{P}_s} \\
    &= \E_{p \sim \Phi_t} \bracket{
      \sum_{a,s'}p(a,s' \mid s)\paren{V^{\pi,p}(s') - \bar{V}^\pi_t(s')}^2 - \V_{a,s' \sim \pi,p}\bracket{V^{\pi,p}(s') - \bar{V}^\pi_t(s')} \given \mathbf{P}_s} \\
    &= \sum_{a,s'}p(a,s' \mid s) \V_{p \sim \Phi_t}\bracket{V^{\pi, p}(s')} - \E_{p \sim \Phi_t}\bracket{\V_{a,s' \sim \pi,p}\bracket{V^{\pi,p}(s') - \bar{V}^\pi_t(s')} \given \mathbf{P}_s}.
  \end{align}
  Applying the outer expectation to the last equation, along with
  \cref{lemma:independence_from_assumptions} and the tower property of expectations yields:
  \begin{equation}
    \label{eq:first_equivalence_total_variance}
    \E[\V[Y|X]] = \sum_{a,s'}\bar{p}_t(a,s' \mid s)\V_{p \sim \Phi_t}\bracket{V^{\pi, p}(s')} - \E_{p \sim \Phi_t}\bracket{\V_{a,s' \sim \pi,p}\bracket{V^{\pi,p}(s') - \bar{V}^\pi_t(s')}}.
  \end{equation}
  Now we repeat the derivation but using $\E[\V[Y|X]] = \E[\E[Y^2|X] - (\E[Y|X])^2]$. For the inner
  expectation of the first term we have:
  \begin{align}
    &\E_{p \sim \Phi_t} \bracket{
      \paren{
        \sum_{a,s'}p(a,s' \mid s)V^{\pi,p}(s')
      }^2 \given \mathbf{P}_s
    }\\
    &= 
    \E_{p \sim \Phi_t} \bracket{
      \sum_{a,s'}p(a,s' \mid s)\paren{V^{\pi,p}(s')}^2 - \V_{a,s' \sim \pi,p}\bracket{V^{\pi,p}(s')} \given \mathbf{P}_s
    }.
  \end{align}
  Applying the outer expectation:
  \begin{equation}
    \label{eq:second_equivalence_total_variance}
    \E[\E[Y^2|X]] = \sum_{a,s'}\bar{p}_t(a,s' \mid s)\E_{p \sim
    \Phi_t}\bracket{\paren{V^{\pi,p}(s')}^2} - \E_{p \sim
    \Phi_t}\bracket{\V_{a,s' \sim \pi,p}\bracket{V^{\pi,p}(s')}}.
  \end{equation}
  Lastly, for the inner expectation of $\E[(\E[Y|X])^2]$:
  \begin{align}
    \paren{\E_{p \sim \Phi_t} \bracket{
        \sum_{a,s'}p(a,s' \mid s)V^{\pi,p}(s')
      \given \mathbf{P}_s}
    }^2 &= 
    \paren{
      \sum_{a,s'}p(a,s' \mid s) \bar{V}^\pi_t(s')
    }^2 \\
    &=
    \sum_{a,s'}p(a,s' \mid s)\paren{\bar{V}^\pi_t(s')}^2 - \V_{a,s' \sim \pi,p}\bracket{\bar{V}^\pi_t(s')}.
  \end{align}
  Applying the outer expectation:
  \begin{equation}
    \label{eq:third_equivalence_total_variance}
    \E[(\E[Y|X])^2] = \sum_{a,s'}\bar{p}_t(a,s' \mid s)\paren{\bar{V}^\pi_t(s')}^2 - \E_{p \sim \Phi_t}\bracket{\V_{a,s' \sim \pi,p}\bracket{\bar{V}_t^{\pi}(s')}}.
  \end{equation}
  Finally, by properties of variance, \cref{eq:first_equivalence_total_variance} =
  \cref{eq:second_equivalence_total_variance} - \cref{eq:third_equivalence_total_variance} which
  gives the desired result.
\end{proof}
\connections*
\begin{proof}
  By definition of $u_t(s)$ in \cref{eq:bellman_exact_reward}, proving the claim is equivalent to
  showing
  \begin{equation}
    \label{eq:thm2_proxy_equivalence}
    \V_{a,s' \sim \pi, \bar{p}_t}\bracket{\bar{V}^\pi_t(s')} = w_t(s) + \E_{p \sim \Phi_t}\bracket{\V_{a,s' \sim \pi, p}\bracket{\bar{V}^\pi_t(s')}},
  \end{equation}
  which holds by combining
  \cref{lemma:pombu_local_uncertainty,lemma:rhs_pombu_uncertainty_first,lemma:rhs_pombu_uncertainty_second}.
  Lastly, $u_t(s) \leq w_t(s)$ holds by \cref{lemma:inflated_uncertainty}.
\end{proof}

\section{THEORY EXTENSIONS}
\subsection{Unknown Reward Function}
\label{app:unknown_rewards}
We can easily extend the derivations on \cref{app:proofs_thm_ube} to include the additional
uncertainty coming from an \emph{unknown} reward function. Similarly, we assume the reward function
is a random variable $r$ drawn from a prior distribution $\Psi_0$, and whose belief will be updated
via Bayes rule. In this new setting, we now consider the variance of the values under the
distribution of MDPs, represented by the random variable $\mathcal{M}$. We need the following
additional assumptions to extend our theory.
\begin{assumption}[Independent rewards]
    \label{assumption:indep_rewards}
    $r(x,a)$ and $r(y,a)$ are independent random variables if $x\neq y$.
\end{assumption}
\begin{assumption}[Independent transitions and rewards]
    \label{assumption:indep_transit_rewards}
    The random variables $p(\cdot \mid s,a)$ and $r(s,a)$ are independent for any $(s,a)$.
\end{assumption}
With \cref{assumption:indep_rewards} we have that the value function of next states is independent
of the transition function and reward function at the current state.
\cref{assumption:indep_transit_rewards} means that sampling $\mathcal{M} \sim \Gamma_t$
is equivalent as independently sampling $p \sim \Phi_t$ and $r \sim \Psi_t$.

\begin{restatable}{theorem}{ube_unknown_reward}
  \label{thm:ube_unknown_rewards}
  Under \crefrange{assumption:transitions}{assumption:indep_transit_rewards}, for any
  $s \in \mathcal{S}$ and policy $\pi$, the posterior variance of the value function, $U_t^\pi =
  \V_{\mathcal{M} \sim \Gamma_t} \bracket{V^{\pi,\mathcal{M}}}$ obeys the uncertainty Bellman equation
  \begin{equation}
    U_t^\pi(s) = \V_{r \sim \Psi_t} \bracket{
      \sum_a \pi(a \mid s) r(s,a)
    } + 
    \gamma ^ 2u_t(s) + \gamma^2\sum_{a, s'}\pi(a \mid s)\bar{p}_t(s' \mid s,a) U_t^\pi(s'),
  \end{equation}
  where $u_t(s)$ is defined in \cref{eq:bellman_exact_reward}.
\end{restatable}
\begin{proof}
  By \cref{assumption:indep_rewards,assumption:indep_transit_rewards} and following the derivation
  of \cref{lemma:variance_decomp_no_assumptions} we have
  \begin{align}
    \V_{\mathcal{M} \sim \Gamma_t} \bracket{V^{\pi, \mathcal{M}}(s)} &= \V_{\mathcal{M} \sim \Gamma_t} \bracket{\sum_a \pi(a \mid s)r(s,a) + \gamma\sum_{a, s'} p(a,s' \mid s)V^{\pi, \mathcal{M}}(s')} \\
    &= \V_{r \sim \Psi_t} \bracket{
      \sum_a \pi(a \mid s) r(s,a)
    } + \V_{\mathcal{M} \sim \Gamma_t} \bracket{\gamma\sum_{a,s'}p(a, s' \mid s)V^{\pi, \mathcal{M}}(s')}.
  \end{align}
  Then following the same derivations as \cref{app:proofs_thm_ube} completes the proof.
\end{proof}

\subsection{Extension to $Q$-values}
\label{app:extension_q_values}
Our theoretical results naturally extend to action-value functions. The following result is
analogous to \cref{thm:ube}.
\begin{restatable}{theorem}{ube_q}
  Under \cref{assumption:transitions,assumption:acyclic}, for any $(s, a) \in \mathcal{S} \times
  \mathcal{A}$ and policy $\pi$, the posterior variance of the $Q$-function, $U_t^\pi = \V_{p \sim
  \Phi_t} \bracket{Q^{\pi,p}}$ obeys the uncertainty Bellman equation
  \begin{equation}
    U_t^\pi(s,a) = 
    \gamma ^ 2u_t(s,a) + \gamma^2\sum_{a', s'}\pi(a' \mid s')\bar{p}_t(s' \mid s,a) U_t^\pi(s', a'),
  \end{equation}
  where $u_t(s,a)$ is the local uncertainty defined as
  \begin{equation}
    u_t(s,a) = \V_{a', s' \sim \pi, \bar{p}_t} \bracket{\bar{Q}_t^\pi(s', a')} -
    \E_{p \sim \Phi_t} \bracket{\V_{a', s' \sim \pi, p} \bracket{Q^{\pi, p}(s',a')}}
  \end{equation}
\end{restatable}
\begin{proof}
  Follows the same derivation as \cref{app:proofs_thm_ube}
\end{proof}
Similarly, we can connect to the upper-bound found by \citet{zhou_deep_2020} with the following
theorem.
\begin{restatable}{theorem}{connections_q_function}
  Under \cref{assumption:transitions,assumption:acyclic}, for any $(s, a) \in \mathcal{S} \times
  \mathcal{A}$ and policy $\pi$, it holds that $u_t(s,a) = w_t(s,a) - g_t(s,a)$, where $w_t(s,a) =
  \V_{p\sim \Phi_t} \bracket{\sum_{a', s'}\pi(a' \mid s') p(s' \mid s,a) \bar{Q}^\pi_t(s',a')}$ and
  $g_t(s,a) = \E_{p \sim \Phi_t}\bracket{\V_{a',s' \sim \pi, p} \bracket{Q^{\pi,p}(s',a')}-
  \V_{a',s' \sim \pi, p} \bracket{\bar{Q}^\pi_t(s',a')}}$. Furthermore, we have that the gap $g_t(s, a)
  \geq 0$ is non-negative, thus $u_t(s,a) \leq w_t(s,a)$.
\end{restatable}
\begin{proof}
  Follows the same derivation as \cref{app:proofs_thm_connections}. Similarly, we can prove that the
  gap $g_t(s,a)$ is non-negative by showing it is equal to $\E_{p \sim
  \Phi_t}\bracket{\V_{a',s' \sim \pi, p} \bracket{Q^{\pi,p}(s',a') - \bar{Q}^\pi_t(s',a')}}$.
\end{proof}

\subsection{State-Action Uncertainty Rewards}
\label{app:q_uncertainty_rewards}
In our practical experiments, we use the results of both
\Cref{app:unknown_rewards,app:extension_q_values} to compose the uncertainty rewards propagated via
the UBE. Concretely, we consider the following two approaches for computing state-action uncertainty
rewards:
\begin{itemize}
  \item \texttt{pombu}: 
  \begin{equation}
    \label{eq:pombu_q_rewards}
    w_t(s,a) = \V_{p\sim \Phi_t} \bracket{\sum_{a', s'}\pi(a' \mid s') p(s' \mid s,a)
    \bar{Q}^\pi_t(s',a')}
  \end{equation}
  \item \texttt{exact-ube}:
  \begin{equation}
    \label{eq:exact_ube_q_rewards}
    u_t(s,a) = w_t(s,a) - \E_{p \sim
    \Phi_t}\bracket{\V_{a',s' \sim \pi, p} \bracket{Q^{\pi,p}(s',a') - \bar{Q}^\pi_t(s',a')}}
  \end{equation}
\end{itemize}

Additionally, since we also learn the reward function, we add to the above the uncertainty term
generated by the reward function posterior, as shown in \cref{app:unknown_rewards}:
$\V_{r \sim \Psi_t} \bracket{r(s,a)}$.

\section{TABULAR ENVIRONMENTS EXPERIMENTS}
In this section, we provide more details about the tabular implementation of
\Cref{algorithm:our_algorithm}, environment details and extended results.

\subsection{Implementation Details}
\label{app:experimental_details}
\paragraph{Model learning.} For the transition function we use a prior
$\text{Dirichlet}(1/\sqrt{S})$ and for rewards a standard normal $\mathcal{N}(0,1)$, as done by
\citet{odonoghue_making_2019}. The choice of priors leads to closed-form posterior updates based on
state-visitation counts and accumulated rewards. We add a terminal state to our modeled MDP in order
to compute the values in closed-form via linear algebra. 

\paragraph{Accelerating learning.} For the \emph{DeepSea} benchmark we accelerate learning by
imagining each experienced transition $(s, a, s', r)$ is repeated $L$ times, as initially suggested
in \citet{osband_deep_2019} (see footnote $9$), although we scale the number of repeats with the
size of the MDP. Effectively, this strategy forces the MDP posterior to shrink faster, thus making
all algorithms converge in fewer episodes. The same strategy was used for all the methods evaluated
in the benchmark.

\paragraph{Policy optimization.} All tested algorithms (PSRL and OFU variants) optimize the policy
via policy iteration, where we break ties at random when computing the $\argmax$, and limit the
number of policy iteration steps to $40$.

\paragraph{Hyperparameters.} Unless noted otherwise, all tabular RL experiments use a discount
factor $\gamma=0.99$, an exploration gain $\lambda = 1.0$ and an ensemble size $N=5$.

\paragraph{Uncertainty reward clipping.} For \emph{DeepSea} we clip uncertainty rewards with
$u_{\min} = -0.05$ and for the 7-room environment we keep $u_{\min} = 0.0$.

\subsection{Environment Details}
\paragraph{\emph{DeepSea}.} As proposed by \citet{osband_deep_2019}, \emph{DeepSea} is a grid-world
environment of size $L \times L$, with $S = L^2$ and $A = 2$.

\paragraph{7-room.} As implemented by \citet{domingues_rlberry_2021}, the 7-room environment
consists of seven connected rooms of size $5\times 5$, represented as an MDP of size $S=181$ and
discrete action space with size $A=4$. The starting state is always the center cell of the middle
room, which yields a reward of $0.01$. The center cell of the left-most room gives a reward of $0.1$
and the center cell of the right-most room gives a large reward of $1$. The episode terminates after
$40$ steps and the state with large reward is absorbing (i.e., once it reaches the rewarding state,
the agent remains there until the end of the episode). The agent transitions according to the
selected action with probability $0.95$ and moves to a randomly selected neighboring cell with
probability $0.05$.

\subsection{\emph{DeepSea} Additional Experiments} 
\subsubsection{Extended Results}
\label{app:extended_results}
\cref{fig:full_results_deep_sea} shows the total regret in intervals of 50 episodes for all the
different \emph{DeepSea} sizes considered. Our method consistently achieves the lowest total regret.

\begin{figure}[t]
	\centering
  \includegraphics{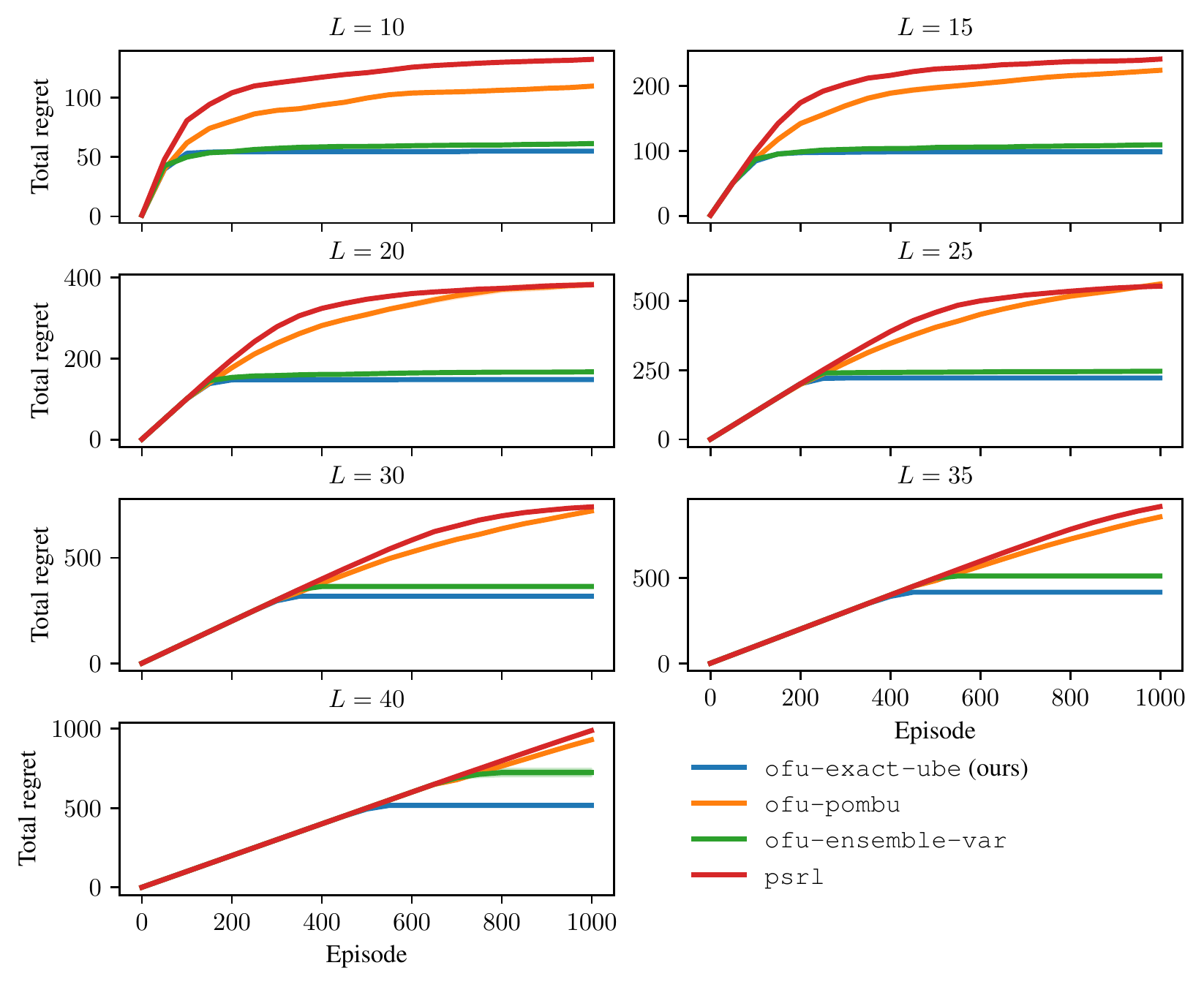}
  \caption{Extended results for the \emph{DeepSea} experiments shown in \cref{fig:deep_sea_results}.
  We report the average (solid line) and standard error (shaded region) over 5 random seeds.}
  \label{fig:full_results_deep_sea}
\end{figure}

\subsubsection{Uncertainty Rewards Ablation}
\label{app:ablation_exact_ube}
Our theory prescribes equivalent expressions for the uncertainty rewards under the assumptions.
However, since it practice the assumptions do not generally hold, the expressions are no longer
equivalent. In this section we evaluate the performance in the \emph{DeepSea} benchmark for these
different definitions of the uncertainty rewards:
\begin{itemize}
  \item \texttt{exact-ube\textunderscore 1}: 
  $$u_t(s,a) = \V_{a', s' \sim \pi,
    \bar{p}_t} \bracket{\bar{Q}_t^\pi(s', a')} - \E_{p \sim \Phi_t} \bracket{\V_{a', s' \sim \pi, p}
    \bracket{Q^{\pi, p}(s',a')}}$$
  \item \texttt{exact-ube\textunderscore 2}: 
  $$u_t(s,a) = \V_{p\sim \Phi_t} \bracket{\sum_{a', s'}\pi(a' \mid s') p(s' \mid s,a) \bar{Q}^\pi_t(s',a')} - \E_{p \sim \Phi_t}\bracket{\V_{a',s' \sim \pi, p} \bracket{Q^{\pi,p}(s',a')}-
  \V_{a',s' \sim \pi, p} \bracket{\bar{Q}^\pi_t(s',a')}}$$
  \item \texttt{exact-ube\textunderscore 3} (labeled \texttt{exact-ube} in all other plots):
  $$u_t(s,a) = \V_{p\sim \Phi_t} \bracket{\sum_{a', s'}\pi(a' \mid s') p(s' \mid s,a)
  \bar{Q}^\pi_t(s',a')} - \E_{p \sim
  \Phi_t}\bracket{\V_{a',s' \sim \pi, p} \bracket{Q^{\pi,p}(s',a') - \bar{Q}^\pi_t(s',a')}}$$
\end{itemize}
Recall that, since we consider an unknown reward function, we add the uncertainty about rewards to
the above when solving the UBE. \Cref{fig:exact_ube_ablation} shows the results for the
\emph{DeepSea} benchmark comparing the three uncertainty signals. Since the assumptions are violated
in the practical setting, the three signals are no longer equivalent and result in slightly
different uncertainty rewards. Still, when integrated into \cref{algorithm:our_algorithm}, the
performance in terms of learning time and total regret is quite similar. We select
\texttt{exact-ube\textunderscore 3} as the default estimate for all other experiments.

\begin{figure}[t]
	\centering
  \includegraphics{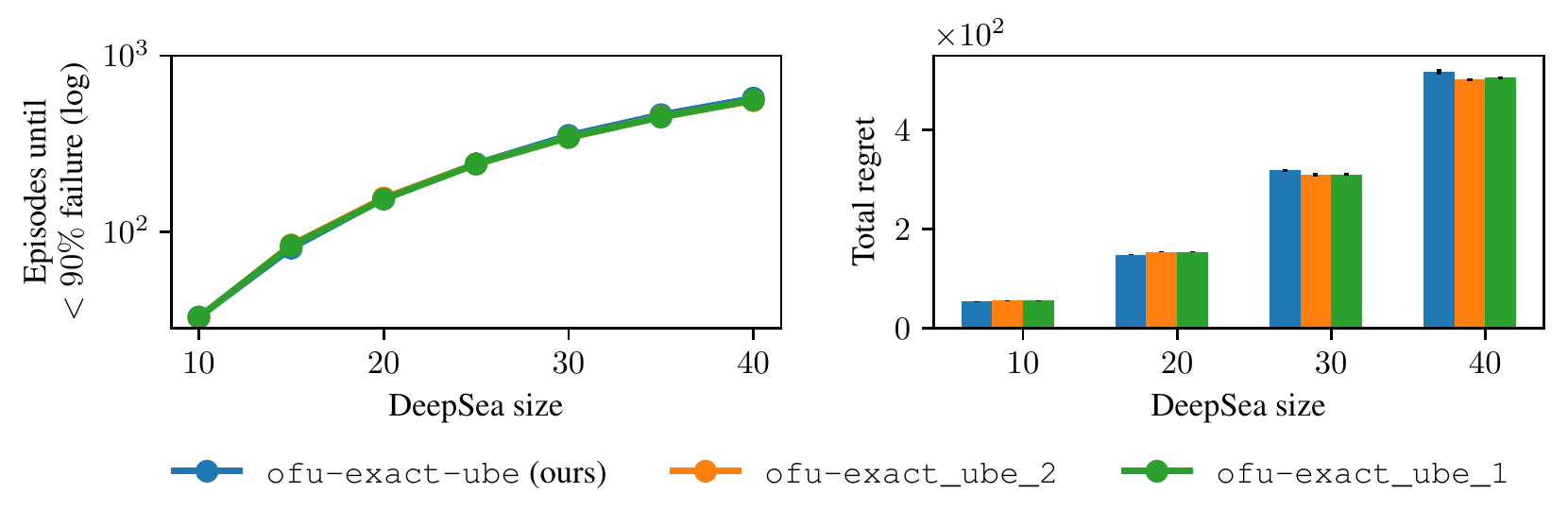}
	\caption{Ablation study on \emph{DeepSea} exploration for different estimates of
  \texttt{exact-ube}. Results represent the average over 5 random seeds, and vertical
  bars on total regret indicate the standard error.}
  \label{fig:exact_ube_ablation}
\end{figure}

\subsubsection{Ensemble Size Ablation}
\label{app:ablation_ensemble_size}
The ensemble size $N$ is one important hyperparameter for all the OFU-based methods. We perform
additional experiments in \emph{DeepSea} for different values of $N$, keeping all other
hyperparameters fixed and with sizes $L = \set{20, 30}$. The results in \Cref{fig:ensemble_size}
show that our method achieves lower total regret across the different ensemble sizes. For
\texttt{ensemble-var}, performance increases for larger ensembles. These results suggest that the
sample-based approximation of our uncertainty rewards is not very sensitive to the number of samples
and achieve good performance even for $N=2$.
\begin{figure}[t]
	\centering
  \includegraphics{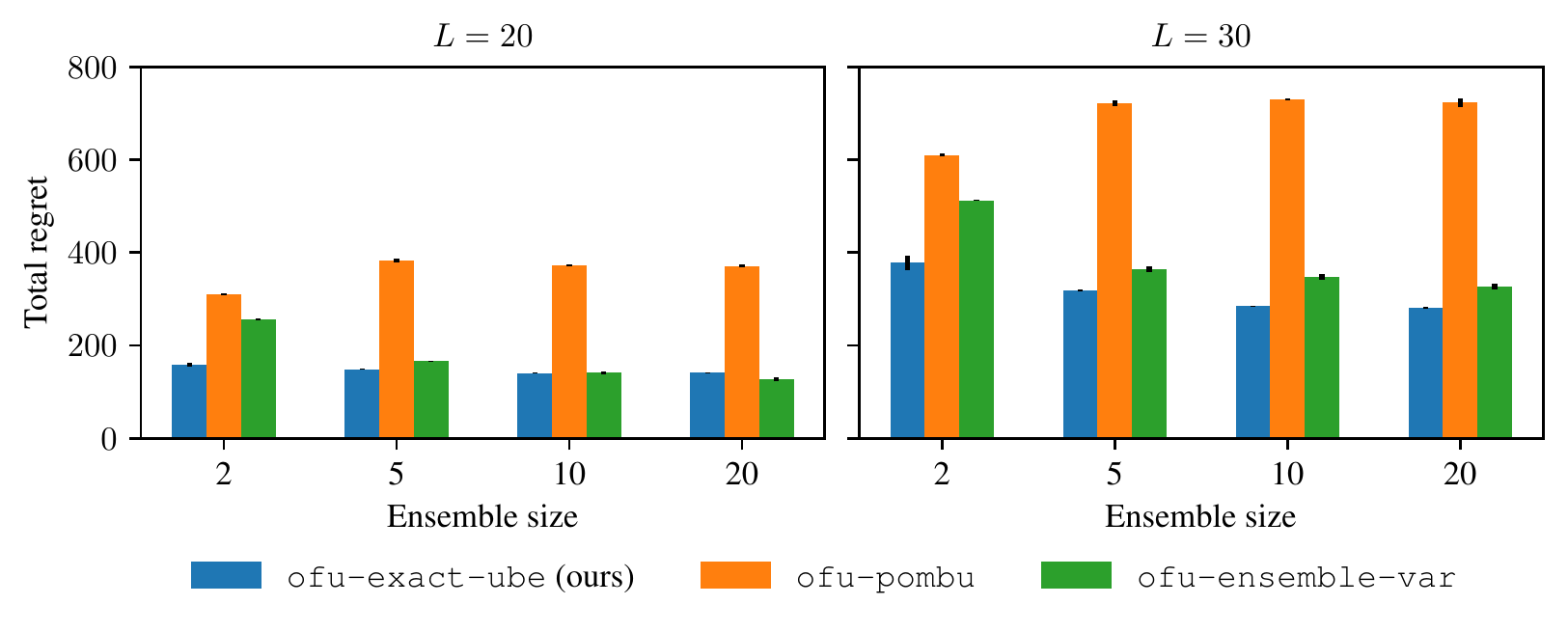}
	\caption{Ablation study over ensemble size $N$ on the \emph{DeepSea} environment.}
  \label{fig:ensemble_size}
\end{figure}

\subsubsection{Exploration Gain Ablation}
\label{app:ablation_exploration_gain}
Another important hyperparameter for OFU-based methods is the exploration gain $\lambda$,
controlling the magnitude of the optimistic values optimized via policy iteration. We perform an
ablation study over $\lambda$, keeping all other hyperparameters fixed and testing for
\emph{DeepSea} sizes $L = \set{20, 30}$. \Cref{fig:exploration_gain} shows the total regret for OFU
methods over increasing gain. Unsurprisingly, as we increase $\lambda$, the total regret of all the
methods increases, but overall \texttt{exact-ube} achieves the best performance.
\begin{figure}[t]
	\centering
  \includegraphics{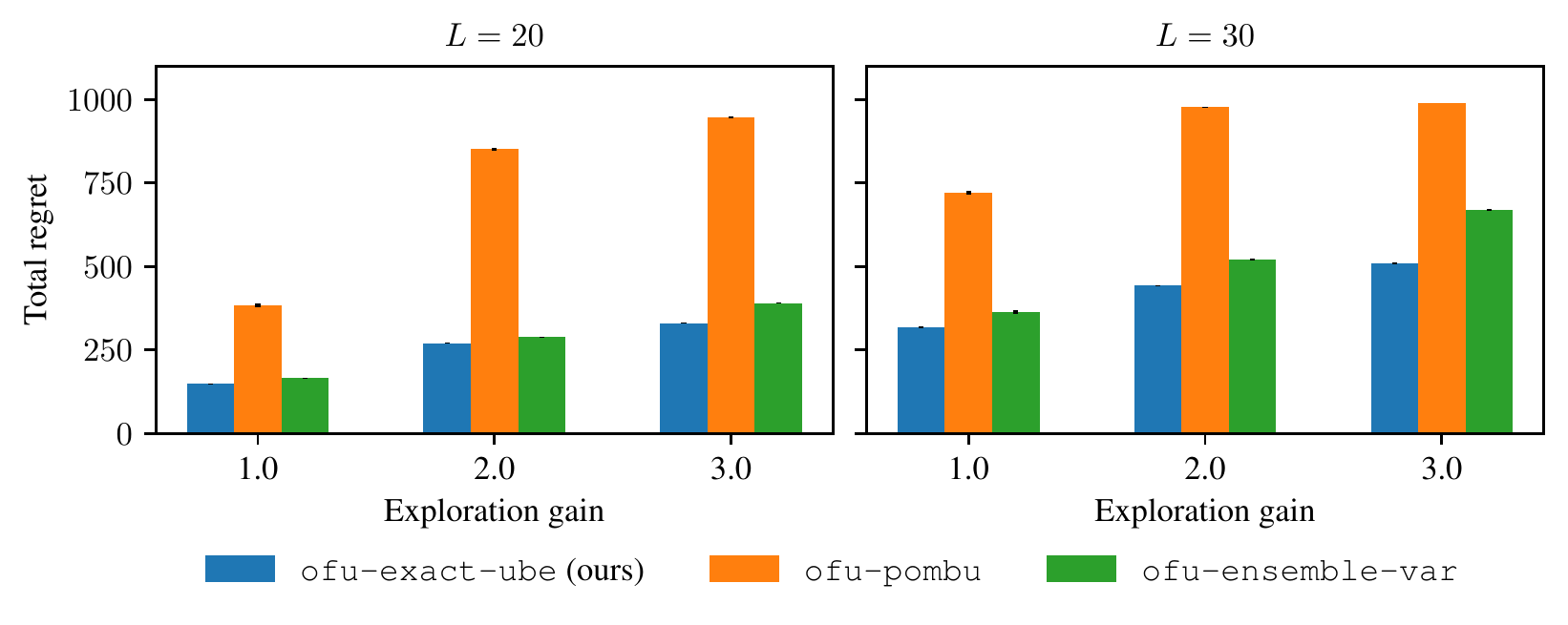}
	\caption{Ablation study over exploration gain $\lambda$ on the \emph{DeepSea} environment.}
  \label{fig:exploration_gain}
\end{figure}

\section{CONTINOUS CONTROL EXPERIMENTS}
In this section, we provide details regarding the deep RL implementation of the optimistic,
variance-driven policy optimization. Also, we include relevant hyperparameters, environment details
and additional results.

\subsection{Implementation Details}
\label{app:deep_rl_implementation}
The optimistic approach on top of MBPO \citep{janner_when_2019} is presented in
\Cref{algorithm:our_algorithm_deep_rl}. The main differences with the original implementation are as
follows:
\begin{itemize}
  \item In \Cref{line:model_rollouts}, we perform a total of $N+1$ $k$-step rollouts corresponding
  to both the model-randomized and model-consistent rollout modalities. The original MBPO only
  executes the former to fill up $\mathcal{D}_{\text{model}}$.
  \item In \Cref{line:q_update}, we update the ensemble of $Q$-functions on the corresponding
  model-consistent buffer. MBPO trains twin critics (as in SAC) on mini-batches from
  $\mathcal{D}_{\text{model}}$.
  \item In \Cref{line:u_update}, we update the $U$-net for the UBE-based variance estimation
  methods.
  \item In \Cref{line:pi_update}, we update $\pi_\phi$ by maximizing the optimistic $Q$-values. MBPO
  maximizes the minimum of the twin critics (as in SAC). Both approaches include an entropy
  maximization term.
\end{itemize}

The main hyperparameters for our experiments are included in \Cref{tab:hparam}. Further
implementation details are now provided. 
\begin{algorithm}[tb]
   \caption{MBPO-style optimistic learning}
   \label{algorithm:our_algorithm_deep_rl}
\begin{algorithmic}[1]
  \STATE Initialize policy $\pi_{\phi}$, predictive model $p_{\theta}$, critic ensemble
  $\set{Q_i}_{i=1}^{N}$, uncertainty net $U_\psi$ (optional), environment dataset $\mathcal{D}_t$,
  model datasets $\mathcal{D}_{\text{model}}$ and $\set{\mathcal{D}^i_{\text{model}}}_{i=1}^{N}$.

  \STATE global step $\leftarrow 0$
  \FOR{episode $t=0, \dots, T-1$}
    \FOR{$E$ steps}
      \IF{global step \% $F == 0$}
        \STATE Train model $p_{\theta}$ on $\mathcal{D}_t$ via maximum likelihood
        \FOR{$M$ model rollouts}
          \STATE Perform $k$-step model rollouts starting from $s \sim \mathcal{D}_t$; add to $\mathcal{D}_{\text{model}}$ and $\set{\mathcal{D}^i_{\text{model}}}_{i=1}^{N}$ \label{line:model_rollouts}
        \ENDFOR
      \ENDIF
      \STATE Take action in environment according to $\pi_{\phi}$; add to $\mathcal{D}_t$
      \FOR{$G$ gradient updates}
        \STATE Update $\set{Q_i}_{i=1}^{N}$ with mini-batches from $\set{\mathcal{D}^i_{\text{model}}}_{i=1}^{N}$, via SGD on \cref{eq:loss_q} \label{line:q_update}
        \STATE (Optional) Update $U_\psi$ with mini-batches from $\mathcal{D}_{\text{model}}$, via SGD on \cref{eq:loss_u} \label{line:u_update}
        \STATE Update $\pi_\phi$ with mini-batches from $\mathcal{D}_{\text{model}}$, via stochastic gradient ascent on the optimistic values of \cref{eq:policy_opt} \label{line:pi_update}
      \ENDFOR
    \ENDFOR
    \STATE global step $\leftarrow$ global step $+ 1$
  \ENDFOR
\end{algorithmic}
\end{algorithm}
\renewcommand{\arraystretch}{1.5}
\begin{table}[H]
\caption{Hyperparameter settings for continuous control experiments.}
\label{tab:hparam}
\begin{center}
\begin{tabular}{c|c|c|c|c}
\textbf{Hyperparameter}  & Sparse Pendulum & HalfCheetah & Walker2D & Ant \\
\hline
$T$ - \# episodes & $75$ & $200$ & \multicolumn{2}{c}{$300$} \\
\hline
$E$ - \# steps per episode & $400$ & \multicolumn{3}{c}{$1000$} \\
\hline
$G$ - policy updates per step & $20$  & \multicolumn{3}{c}{$10$}\\
\hline
$M$ - \# model rollouts per step & \multicolumn{4}{c}{$400$}\\
\hline
$F$ - frequency of model retraining (\# steps) & $400$  & \multicolumn{3}{c}{$250$}\\
\hline
retain updates & $1$  & \multicolumn{3}{c}{$10$}\\
\hline
$N$ - ensemble size & \multicolumn{4}{c}{$5$}\\
\hline
$\lambda$ - exploration gain & \multicolumn{4}{c}{$1.0$} \\
\hline
$\lambda_{\text{reg}}$ - UBE regulatization gain & $5.0$ & \multicolumn{3}{c}{$0.0$} \\
\hline
$k$ - rollout length & $10$  & \multicolumn{3}{c}{$1$}\\
\hline
Model network & \multicolumn{4}{c}{$4$ layers, $200$ units, SiLU activations}\\
\hline
Policy network & $2$ layers, $64$ units, Tanh activation & \multicolumn{3}{c}{$2$ layers, $128$
units, Tanh activations}\\
\hline
$Q$ and $U$ networks & \multicolumn{4}{c}{$2$ layers, $256$ units, Tanh activations}
\end{tabular}
\end{center}
\end{table}

\paragraph{Model learning.} We leverage the \texttt{mbrl-lib} Python library from
\citet{pineda_mbrl-lib_2021} and train an ensemble of $N$ probabilistic neural networks. We use the
default MLP architecture with four layers of size 200 and SiLU activations. The networks predict
delta states, $\Delta = s' - s$, and receive as input normalized state-action pairs. The
normalization statistics are updated each time we train the model, and are based on the dataset
$\mathcal{D}_t$. We use the default initialization of the network provided by the library, which
samples weights from a truncated Gaussian distribution, however we found it helpful to increase by a
factor of $2.0$ the standard deviation of the truncated Gaussian for the sparse pendulum task; a
wider distribution of weights allows for more diverse dynamic models at the beginning of training
and thus a stronger uncertainty signal to guide exploration.

\paragraph{Model-generated buffers.} The capacity of the model-generated buffers
$\mathcal{D}_{\text{model}}$ and $\set{\mathcal{D}^i_{\text{model}}}_{i=1}^{N}$ is computed as $k
\times M \times F \times$\texttt{retain updates}, where \texttt{retain updates} is the number of
model updates before entirely overwriting the buffers. Larger values of this parameter allows for
more off-policy (old) data to be stored and sampled for training.

\paragraph{Uncertainty reward estimation.} We estimate the uncertainty rewards
\cref{eq:pombu_q_rewards,eq:exact_ube_q_rewards} using a finite-sample approximation. For
$w_t(s,a)$, the inner expectation is estimated using a single action $a' \sim \pi(\cdot \mid s')$,
where we take $s'$ to be the mean of the Gaussian distribution parameterized by each ensemble
member. For the gap term in $u_t(s,a)$, we sample $10$ actions from the current policy to estimate
the aleatoric variance term $\V_{a',s'}\bracket{\cdot}$. We clip the uncertainty rewards with
$u_{\min} = 0.0$.

\paragraph{SAC specifics.} Our SAC implementation is based on the open-source repository
\url{https://github.com/pranz24/pytorch-soft-actor-critic}, as done by \texttt{mbrl-lib}. For all
our experiments, we use the automatic entropy tuning flag that adaptively modifies the entropy gain
$\alpha$ based on the stochasticity of the policy.

\subsection{Environment Details}
\paragraph{Sparse Pendulum.} The implementation is taken from
\url{https://github.com/sebascuri/hucrl/blob/4b4446e54a7269366eeafabd90f91fbe466d8b15/exps/inverted_pendulum/util.py}
and adapted to the OpenAI Gym \citep{brockman_openai_2016} convention for RL environments. We use an
action cost multiplier $\rho = 0.2$ for all our experiments.

\paragraph{Pybullet environments.} We use the default Pybullet locomotion environments but remove
the observations related to feet contact, which are represented as binary variables, as these can
pose challenges to model learning.

\subsection{Ensemble Size Ablation}
\label{app:sparse_pendulum_ensemble_ablation}
We repeat the experiments for the sparse pendulum task for different ensemble sizes, and summarize
the results in \cref{fig:sparse_pendulum_ensemble_ablation}. In most cases, performance increases
with $N$, although there exists some outliers. In some specific cases, we observed larger ensembles
could be detrimental to learning with sparse rewards: if most members of the ensemble converge to
similar values then the policy might prematurely converge to a suboptimal policy. We believe network
initialization and regularization may play a critical role in maintaining sufficient ensemble
diversity to drive exploration in sparse reward settings.

\begin{figure}[ht]
	\centering
  \includegraphics{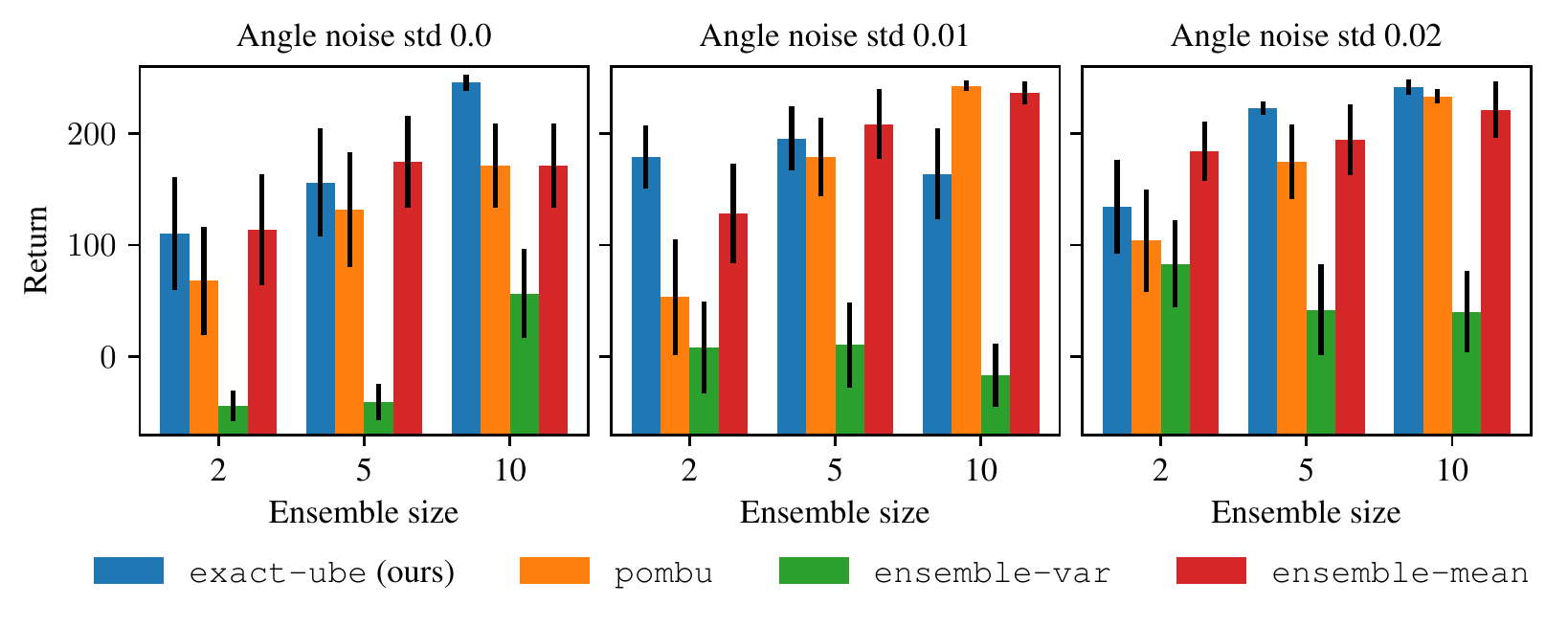}
  \caption{Ensemble size ablation study on the sparse pendulum swing-up problem. We report the mean and standard error of the final return after 75 episodes over 10 random seeds.}
  \label{fig:sparse_pendulum_ensemble_ablation}
\end{figure}

\subsection{Visualization of Variance Estimates}
\label{app:pendulum_viz}
In this section, we visualize the evolution of the value function and variance estimates during
training in the sparse pendulum problem using optimistic values estimated with the
\texttt{exact-ube} method. In \cref{fig:pendulum_viz}, we plot the mean $Q$-values and the standard
deviations corresponding to the \texttt{exact-ube} and \texttt{ensemble-var} estimates. While both
\texttt{exact-ube} and \texttt{ensemble-var} have higher variance in regions of interest for
exploration, the latter outputs much larger estimates, which may lead to over-exploration. 
\begin{figure}[ht]
	\centering
  \includegraphics{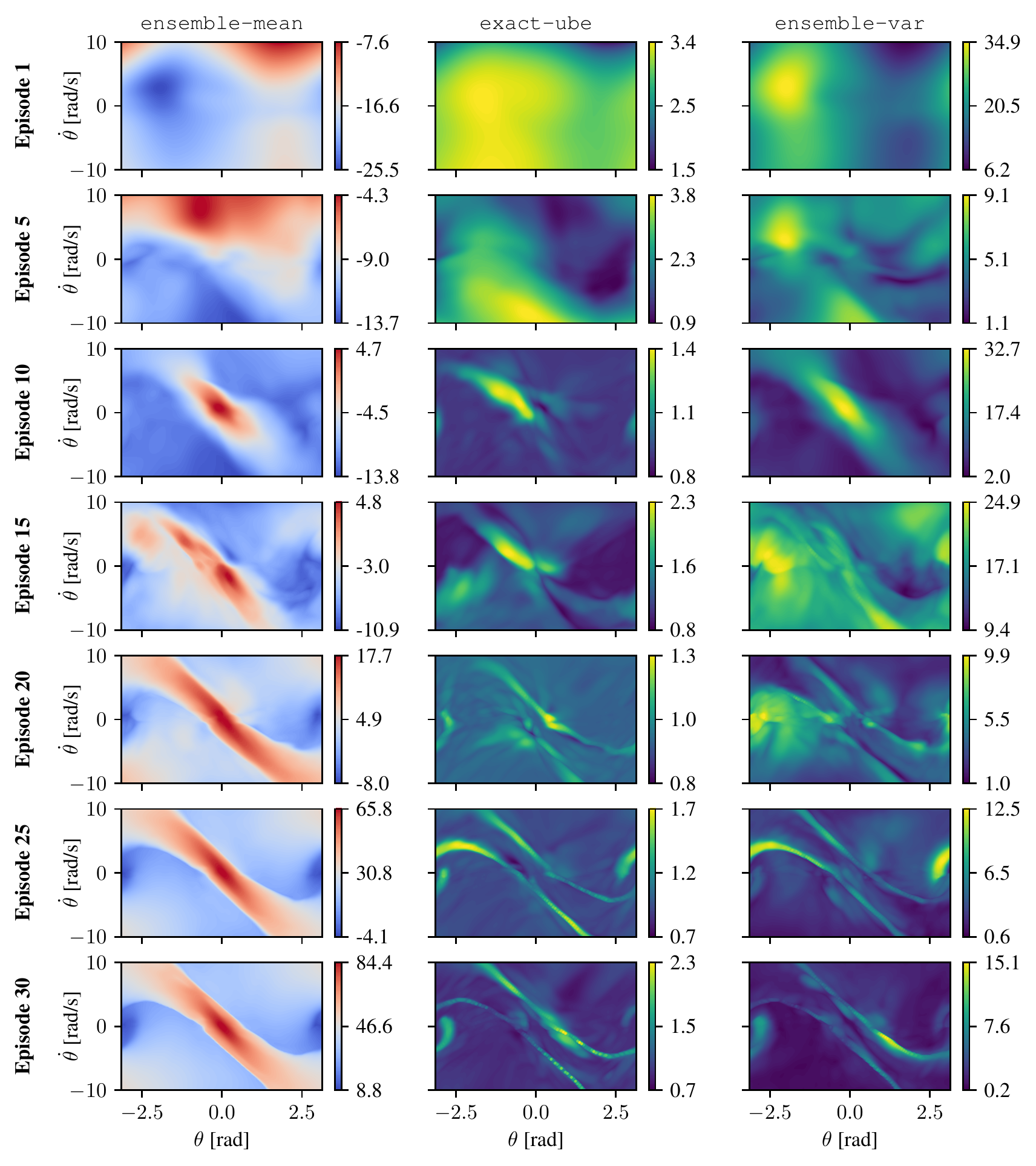}
  \caption{Visualization of training in sparse pendulum swing-up task using optimistic values
  estimated with the \texttt{exact-ube} method. (Left column) The mean values correspond to
  $\bar{Q}^\pi_t(s, \bar{a})$, where $\bar{a}$ is the mean of the Gaussian policy $\pi$ at the
  corresponding episode. (Center and right columns) The posterior standard deviation of $Q$-values,
  computed as $\sqrt{\hat{U}^\pi_t(s, \bar{a})}$ for the \texttt{exact-ube} and
  \texttt{ensemble-var} variance estimates.}
  \label{fig:pendulum_viz}
\end{figure}

\end{document}